\documentclass[paper=a4, fontsize=11pt]{article}
\usepackage{mathtools}
\usepackage{amsfonts}
\usepackage{amsthm}
\usepackage{amssymb}
\usepackage{amsmath}

\usepackage{graphicx,import}
\usepackage{epstopdf}  

\usepackage[titletoc,title]{appendix}
\usepackage[usenames]{xcolor} 
\usepackage{my-ieee-macros}
\usepackage{algorithmic} 
\usepackage{algorithm}

\newcommand{\vect}[1]{\textbf{#1}}

\usepackage{caption}
\usepackage{subcaption}

\usepackage{dblfloatfix} 

\usepackage{url}
\usepackage{hyperref}
\hypersetup{
pdfauthor = {R. Gill and D. Kuli\'{c} and C. Nielsen},
pdftitle = {Spline Path Following for Redundant Mechanical Systems},
pdfdisplaydoctitle = {true},
unicode = {false},
pdfsubject = {Nonlinear Systems},
pdfkeywords = {Path following, set stabilization, redundancy resolution},
pdfcreator = {\LaTeX with \flqq hyperref \flqq package},
pdfproducer = {pdflatex},
bookmarksnumbered,
pdfstartview={FitV},
colorlinks = true, 
linkcolor = blue,
anchorcolor = red,
citecolor = red, 
filecolor = blue, 
urlcolor = red
}  

\begin{document}
\title{Spline Path Following for Redundant Mechanical Systems}
\urldef{\nielsen}\url{cnielsen@uwaterloo.ca}
\urldef{\rg}\url{rjgill@uwaterloo.ca}
\urldef{\dk}\url{dana.kulic@uwaterloo.ca}
\urldef{\combined}\url{{rjgill; dana.kulic; cnielsen}@uwaterloo.ca}
\author{R.J.~Gill \qquad D. Kuli\'{c} \qquad C.~Nielsen
  \thanks{Supported by the Natural Sciences and Engineering Research
    Council of Canada (NSERC).}
  \thanks{The authors are with the Department of Electrical and
    Computer Engineering, University of Waterloo, 200 University
    Avenue West, Waterloo, Ontario, Canada, N2L 3G1. \combined}
}

\markboth{Submitted for Review}%
{Gill, Kulic, Nielsen : Spline Path Following for Redundant Euler-Lagrange
  Systems}

\maketitle

\begin{abstract}
  Path following controllers make the output of a control system
  approach and traverse a pre-specified path with no \emph{apriori}
  time parametrization. In this paper we present a method for path
  following control design applicable to framed curves generated by
  splines in the workspace of kinematically redundant mechanical
  systems. The class of admissible paths includes self-intersecting
  curves. Kinematic redundancies are resolved by designing controllers
  that solve a suitably defined constrained quadratic optimization
  problem. By employing partial feedback linearization, the proposed
  path following controllers have a clear physical
  meaning. 
  The approach is experimentally verified on a 4-degree-of-freedom
  (4-DOF) manipulator with a combination of revolute and linear
  actuated links and significant model uncertainty.
\end{abstract}


\section{Introduction}




The problem of designing feedback control laws that make the output of a
control system follow a desired curve in its workspace can be broadly
classified as either a trajectory tracking or path following
problem. As previously suggested in \cite{hauser1995maneuver}, trajectory tracking consists of
tracking a curve with an assigned time parameterization. In other
words, besides the shape of the curve, the end-effector's motion has a
desired velocity and acceleration profile. This
approach may not be suitable for certain applications as it may limit
the attainable performance, be unnecessarily demanding or even
infeasible. For example, the primary objective of a robotic
  manipulator tracing a contour in its output space in
  welding~\cite{sicard1988approach} or
  cutting~\cite{erkorkmaz2001high} applications is to accurately
  traverse a path. The trajectory tracking approach parametrizes the
  desired motion along the path and a controller is designed to
  asymptotically stabilize the tracking error to zero. When the
  tracking error is non-zero, the controller may drive the
  end-effector to ``cut through'' an obstacle in order to reduce the
  tracking error and thereby not precisely follow the contour.

On the other hand, path following or manoeuvre regulation controllers, as previously defined in \cite{hauser1995maneuver}, drive a system's output
to a desired path and make the output traverse the path without a
pre-specified time parametrization. Thus, a path following controller
stabilizes all trajectories that cause the system's output to traverse
the desired path, whereas trajectory tracking controllers do this for
a single trajectory. In general, path following results in smoother
convergence to the desired path compared to trajectory tracking, and
the control signals are less likely to
saturate~\cite{Lapierre20071734}. It is also shown in~\cite{MilMid08}
for the linear case, and~\cite{aguiar2005path} for the nonlinear case,
that for non-minimum phase systems, path following controllers remove
performance limitations compared to trajectory tracking controllers,
which have a lower bound on the achievable $\mathcal{L}_2$-norm of the
tracking error. 

There are several approaches for implementing path following. The
papers \cite{hauser1995maneuver,skjetne2004robust,kant1986toward} project the current system state onto the desired
path in real time. Then a tracking controller is used to stabilize
this generated reference point. Contour error mitigation controllers
have also been proposed in the machining community in addition to a
trajectory tracking controller in order to drive the machine tool
towards a 
path~\cite{chiu2001contouring,chuang1991cross,yang2007contour}. However,
these approaches do not guarantee that the desired path will be
invariant, i.e., if for some time the output is on the path and has a
velocity tangent to the path, then the output will remain on the path
for all future time.

Virtual holonomic constraints can be used to guarantee invariance of a
path via set stabilization. This methodology is used for orbital
stabilization for a Furuta pendulum~\cite{shiriaev2007virtual},
helicopters~\cite{westerberg2009motion}, and, in the sense of hybrid
systems, for the walking/running of bipedal
robots~\cite{chevallereau2004asymptotic}. Also,
\cite{banaszuk1995feedback} introduces conditions for feedback
linearization of the transverse dynamics associated with periodic
orbits. This inspired the work in~\cite{nielsen2006output} in which an
arbitrary non-periodic path in the output space can be made attractive
and invariant using transverse feedback linearization (TFL) for
single-input systems. This idea was further extended
in~\cite{TFL:hladio2011path} by providing conditions for which desired
motions can be achieved while on the path, as opposed to orbital stabilization where the dynamics along the closed curve in the state space are fixed. 
The approach
in~\cite{TFL:hladio2011path} simplifies controller design and
decomposes the path following design problem into two independent
sub-problems: staying on the path and moving along the path.

However, in \cite{TFL:hladio2011path}, perfect knowledge of the
dynamic parameters of the system was assumed, performance was
demonstrated only for a linear, decoupled system, and the path is
restricted to a special case of simple curves such as circles. One
contribution of this paper is that we enlarge the class of allowable
paths, allowing self-intersecting paths and curves that are generated
through spline interpolation of given waypoints. This is accomplished
by designing a numerical algorithm and by using a systematic approach
to generate zero level set representations of the curves. Another
contribution of this work is extending the class of systems
of~\cite{TFL:hladio2011path} to mechanical systems with kinematic
redundancies, by employing a novel redundancy resolution technique at
the dynamic level, yielding bounded internal dynamics. Khatib proposed a similar redundancy resolution mechanism for a specific class of mechanical systems
in~\cite{Khatib1987unified}, namely the torque-input model of robot
manipulators. However the approach in this paper is to resolve the
redundancies directly in the context of path following via feedback
linearization, unlike in~\cite{Khatib1987unified}, and is thus
applicable to general mechanical systems, which may include a dynamic map between the system input and the output forces. To the best of the
author's knowledge, redundancy resolution in the context of feedback
linearization is an unsolved problem. We show that, under the
assumption that the redundant system has viscous friction on each
degree of freedom, the control law obtained by solving a static,
constrained, quadratic optimization problem yields bounded closed-loop
internal dynamics. We further validate the control strategy on a
non-trivial redundant platform with significant modelling
uncertainties using a robust controller.

In this paper, the class of paths that can be followed are broadened, made possible by systematic approach to generating zero level set of a function, and using a numerical algorithm in the controller. A novel approach to resolving redundancies is proposed. Finally, non-trivial experimental verification including robustification is performed. 

The detailed problem statement and proposed approach are outlined in
Section~\ref{section:problem}. The control design is explained in
Section~\ref{section:path_following_design} and experimental results
are shown in Section~\ref{section:Experiment}.

\subsection{Notation} 
Let $\inner{x}{y}$ denote the standard inner product of vectors $x$
and $y$ in $\Real^n$. The Euclidean norm of a vector and induced
matrix norm are both denoted by $\norm{\cdot}$. The notation
$s\circ h : A \rightarrow C$ represents the composition of maps
$s: B \to C$ and $h : A \to B$. The $i$th element of a vector is
denoted $[x]_i$. The symbol $\coloneqq$ means equal by
definition. Given a $C^1$ mapping $\phi : \Real^n \to \Real^m$ let
$\D\phi_x$ be its Jacobian evaluated at $x\in\Real^n$. If $f$,
$g : \Real^n \to \Real^n$ are smooth vector fields we use the
following standard notation for iterated Lie derivatives
$L^0_f\phi \coloneqq \phi$,
$L^k_f\phi \coloneqq L_f(L^{k-1}_f\phi) = \langle \D
L_f^{k-1}\phi_x,f(x)\rangle$,
$L_gL_f\phi \coloneqq L_g(L_f\phi) = \langle \D
L_f\phi_x,g(x)\rangle$.


\section{Problem Formulation} \label{section:problem}

We consider path following problems in the workspace of kinematically
redundant, fully actuated, Euler-Lagrange systems. In this section we
define the class of systems considered, the class of allowable paths
and conclude by stating the problem studied.

\subsection{Class of systems} \label{section:class_of_systems}
We consider a fully actuated Euler-Lagrange system with an
$N$-dimensional configuration space contained in $\Real^N$ and $N$ control
inputs $u \in \Real^N$. The model is given by
\[
\frac{\D}{\D t} \DER{\mathcal{L}}{\dot{q}} - \DER{\mathcal{L}}{{q}} = A(q) u
\]
where $q,\dot{q} \in \Real^N$ are the configuration positions and velocities, respectively, of the system, and $\mathcal{L}(q, \dot{q})$ is the Lagrangian function. We assume
that $\mathcal{L}$ is smooth and has the form
$\mathcal{L}(q, \dot{q}) = K(q, \dot{q}) - P(q)$ where
$K(q, \dot q) = (1/2)\dot{q}^\top D(q)\dot{q}$ is the system's kinetic
energy and $P : \Real^N \rightarrow \mathbb{R}$ is the system's
potential energy. The inertia matrix $D(q)$ is positive definite for
all $q$. Furthermore, $A : \Real^N \to \Real^{N\times N}$, the input matrix (usually identity), is smooth and nonsingular for all $q$. 

The system can be rewritten in the standard vector form 
\begin{equation}
D(q)\ddot{q} + C(q,\dot{q})\dot{q} + G(q) = A(q) u
\label{eq:euler_lagrange_system}
\end{equation}
where $C(q,\dot{q}) \in \mathbb{R}^{N \times N}$ represents the
centripetal and Coriolis terms, and $G(q) = (\D P_q)^\top = \nabla
P(q)\in \mathbb{R}^N$ represents the gravitation effects
\cite{spong2006robot}. 
Defining $x_c \coloneqq q$, $x_v \coloneqq \dot{q}$, $n \coloneqq 2N$,
and $x \coloneqq (x_c, x_v) \in \Real^N \times \Real^N$ we can
express~\eqref{eq:euler_lagrange_system} in state model form
\begin{align} 
	\begin{split}
		\label{eq:system_StateSpace}
		\dot{x} &= f(x) + g(x)u \coloneqq \begin{bmatrix}
		x_v \\
		f_v(x)
		\end{bmatrix} + \begin{bmatrix}
		0 \\
		g_v(x_c)
		\end{bmatrix}u,
	\end{split}
\end{align}
where $f : \Real^n \rightarrow \Real^n$ and
$g : \Real^n \rightarrow \Real^{n \times n}$ are smooth,
$g_v(x_c) \coloneqq D^{-1}(x_c)A(x_c)$ and
$f_v(x) \coloneqq -D^{-1}(x_c) \left( C(x)x_v + G(x_c) \right)$.

Motivated by the task space of a robotic manipulator, we take the output of system~\eqref{eq:system_StateSpace} to satisfy
\begin{equation}
	\label{eqn:OutputOfSystem}
	y = h(x), \quad \frac{\partial h(x)}{\partial x_v} = 0, \quad
        y \in \mathbb{R}^p
\end{equation}
where $p$ is the dimension of the output space and
$h:\Real^n \rightarrow \Real^p$ is smooth. Let
$J(x_c) \coloneqq \frac{\partial h}{\partial x_c}$, the output
Jacobian, be non-singular. We allow the system to be redundant in the
following sense.
\begin{definition}
  System \eqref{eq:system_StateSpace},\eqref{eqn:OutputOfSystem} is
  \textbf{redundant} if $N > p$.
\end{definition}

\subsection{Admissible paths} \label{section:path_assumptions}

We assume that we are given a finite number, $\nsplines \in
\mathbb{N}$, of $(p+1)$-times continuously differentiable, parametrized
curves in the output space of~\eqref{eq:system_StateSpace}
\begin{align}
  \begin{split} \label{eq:path_definition} \sigma_k : \Real
    &\rightarrow \Real^p, \quad
    k \in \set{1,2,\ldots,\nsplines}. \\
\end{split}
\end{align}
The curves need not be polynomials, but we nevertheless use the term
spline in accordance with common practice. We are also given
$\nsplines$ non-empty, closed, intervals of the real line
$\mathbb{I}_k \coloneqq [\lambda_{(k, \mathrm{min})}, \lambda_{(k,
  \mathrm{max})}] \subset \Real$. The desired path is the set
\begin{equation}
  \label{eq:path_image}
\mathcal{P} \coloneqq \bigcup_{k=1}^{\nsplines} \sigma_k\left(\mathbb{I}_k\right).
\end{equation}
We assume that the overall path is $(p+1)$-times continuously
differentiable in the following sense:

\begin{assumption}[smoothness of path]
\label{assumption:path_smooth}
The curves~\eqref{eq:path_definition} satisfy
  \begin{align*} \begin{split}
    \label{eq:path_smooth}
    \left(\forall k \in \set{1,2,\ldots,\nsplines - 1} \right)
    \left(\forall i \in \set{0,1,\ldots,p+1} \right) \\ \left.\frac{\D^i
        \sigma_k}{\D\lambda^i}\right|_{\lambda = \lambda_{(k, \mathrm{max})}} = \left.\frac{\D^i
        \sigma_{k+1}}{\D\lambda^i}\right|_{\lambda = \lambda_{(k+1, \mathrm{min})}}.
	\end{split}
  \end{align*}
  In the case of closed paths, the same equality must hold between spline index $1$ and $\nsplines$.
\end{assumption}

This assumption is required to ensure continuous control effort over the entire path. In the case of spline interpolation (fitting polynomials between desired waypoints),
Assumption~\ref{assumption:path_smooth} can be enforced by setting it as
a constraint in the spline fitting formulation and using polynomials of sufficiently high degree. Typically the
parameter $\lambda$ of each curve $\sigma_k$ represents the chord
length between waypoints; namely, $\lambda_{(k,\mathrm{min})} = 0$ and
$\lambda_{(k,\mathrm{max})} = l_k$ where $l_k$ is the Euclidean
distance between the $k$'th and $(k+1)$
waypoints~\cite{Erkorkmaz2001trajandspline}. 

\begin{remark}
  Assumption~\ref{assumption:path_smooth} can be relaxed if continuity
  of the control signal over the entire path isn't
  required. Furthermore, if Assumption \ref{assumption:path_smooth} is
  only true for $i=0$, i.e., the path is not continuously
  differentiable, the control input can still be kept continuous by
  designing the desired tangential velocity profile in such a way that
  the output slows down and stops at
  $\sigma_k(\lambda_{(k,\mathrm{max})}) =
  \sigma_{k+1}(\lambda_{(k+1,\mathrm{min})})$,
  and then resumes motion along curve $k+1$ in finite time. See Remark
  \ref{remark:eta2ref_trajectory_discontinuous}
\hfill $\bullet$
\end{remark}

We further assume that the curves~\eqref{eq:path_definition} are
framed.
\begin{assumption}[framed curves]
\label{assumption:path_framed}
For each spline $k$, the first $p$ derivatives of $\sigma_k$ are
linearly independent, i.e.,
\begin{equation*}
	\left(\forall k \in \set{1,2,\ldots,\nsplines} \right) \left(\forall
	  \lambda \in \mathbb{I}_k\right) 
	\Sp_{\Real}\set{\sigma_k^\prime(\lambda),\sigma_k^{\prime\prime}(\lambda),\ldots,\sigma_k^{(p)}(\lambda)} = \Real^p.
\end{equation*}
\end{assumption}
Assumption \ref{assumption:path_framed} allows the use of Gramm-Schmidt orthogonalization to construct Frenet-Serret frames (FSFs) in the output space $\Real^p$. The use of
FSFs for control of mobile robots is a well-known
technique~\cite{WitSicBas97}. In this paper we use
Assumption~\ref{assumption:path_framed} to generate a zero-level set
representation of $\mathcal{P}$ in the state space
of~\eqref{eq:system_StateSpace}. Previous works \cite{TFL:hladio2011path} relied on the zero-level set representation to already be available by restricting the class of paths to one-dimensional embedded submanifolds, which precludes self-intersecting curves.


\begin{remark} \label{remark:assumption_framed_relax}
Assumption \ref{assumption:path_framed} can be relaxed to only requiring that the curve be regular by not using FSF to construct a basis that aligns with the curve at each point along the curve. Consider a straight line $\sigma(\lambda) = \begin{bmatrix} \lambda & 0 \end{bmatrix}^\top \in \Real^2$. This curve violates Assumption \ref{assumption:path_framed} at all points $\lambda$. The first basis vector is taken to be the unit tangent vector $\begin{bmatrix} 1 & 0 \end{bmatrix}^\top$(as is the case in the FSF). The normal vector cannot be constructed via FSF, since the curvature is zero. However, one can choose the normal vector to be $\begin{bmatrix}
0 & 1 \end{bmatrix}^\top$. 

Furthermore, consider a regular torsion free curve in the $xy$-plane
of a 3-dimensional workspace. This curve also violates Assumption
\ref{assumption:path_framed}, since $\sigma^{(3)}(\lambda)$ will be a
linear combination of $\sigma'$ and $\sigma''$. However, one can
choose the first basis vector to again be the unit tangent vector
$\frac{\sigma'(\lambda)}{\norm{\sigma'(\lambda)}}$. The unit normal
vector can be chosen to be the 90 degree rotation of the first basis
vector, and the final basis vector can be chosen to be
$\begin{bmatrix} 0 & 0 & 1 \end{bmatrix}^\top$.  Once the appropriate frame is chosen, the proposed approach for coordinate transformation and control design (Section \ref{section:path_following_design}) can be applied without any additional modifications.
\hfill $\bullet$
\end{remark}


Path following control design can naturally be cast as a set
stabilization problem. This point of view was successfully taken
in~\cite{gill2013robust, TFL:hladio2011path}. We also take this
approach but, unlike previous work, we do not require that the path
$\mathcal{P}$ be an embedded, nor immersed, submanifold of the output
space. In other words the curve may be self-intersecting. Removing
this restriction is particularly useful because it allows the curve to
go through an arbitrary set of waypoints in the output space of
system~\eqref{eq:system_StateSpace} and can be constructed using
classical spline fitting.



\subsection{Problem statement}
\label{section:problem_statement}
The problem studied in this paper is to find a continuous feedback
controller for system~\eqref{eq:system_StateSpace}, with $\nsplines$
curves that satisfy Assumptions~\ref{assumption:path_smooth}
and~\ref{assumption:path_framed}, that renders the path $\mathcal{P}$
in~\eqref{eq:path_image} invariant and attractive. Invariance is
equivalent to saying that if for some time $t=0$ the state $x(0)$ is
appropriately initialized with the output $y=h(x(0))$ on the path
$\mathcal{P}$, then $(\forall t \geq 0) \; h(x(t)) \in \mathcal{P}$.
Attractiveness is equivalent to ensuring that for initial
conditions $x(0)$ such that the output $h(x(0))$ is in a neighbourhood of the desired path $\mathcal{P}$, $h(x(t))$ under~\eqref{eq:system_StateSpace}
approaches $\mathcal{P}$.

The motion along the path should also be controllable, so that a
tangential position or velocity profile ($\eta^\mathrm{ref}(t)$) can be tracked. Furthermore,
any remaining redundant dynamics (will show up as internal dynamics
for the closed-loop system) must be bounded.


\section{Multiple Spline Path Following
  Design} \label{section:path_following_design} We begin by defining a
coordinate transformation in Section \ref{section:coordinate_transform} that depends on a given path
segment~\eqref{eq:path_definition} and the system
dynamics~\eqref{eq:system_StateSpace}. These coordinates partition the
system dynamics into a sub-system that governs the motion transversal
to the desired path and a tangential sub-system that governs the
motion of the system along the path (Section \ref{sec:dynamics_and_control}). The tangential dynamics are
further decomposed to reveal the redundant dynamics. It will be shown
that the system in transformed coordinates has a structure that
facilities control design, namely the dynamics will appear linear and
decoupled. One of the keys to this construction is the ability to
unambiguously determine the spline that is currently closest to the
system output~\eqref{eqn:OutputOfSystem}. By doing so, we are able to
deal with self-intersecting curves and curves~\eqref{eq:path_image}
that aren't embedded submanifolds. A novel redundancy resolution approach is then proposed to ensure the redundant dynamics remain bounded (Section \ref{section:redundancy_resln}). The overall block diagram can be found below in Figure \ref{fig:block_diagram}.

\begin{figure}[h!]
	\centering
	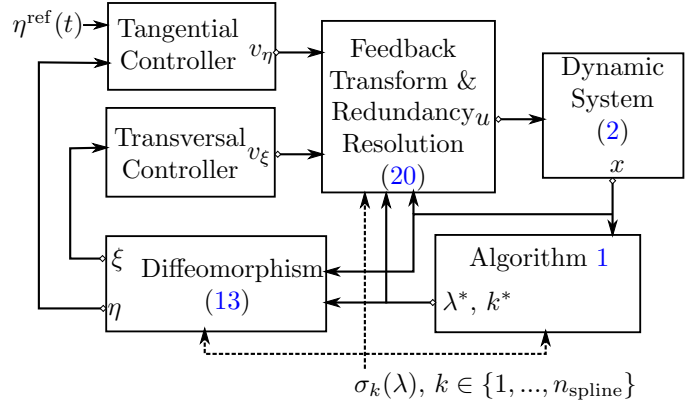
	\caption{Entire control system block diagram.}
	\label{fig:block_diagram}
\end{figure}

%

\subsection{Coordinate Transformation}
\label{section:coordinate_transform} 
We construct a coordinate transformation $T$ that maps the state $x$
of system \eqref{eq:system_StateSpace} to a new set of coordinates
that correspond to the tangential position/velocity along the path,
and the distance and velocity towards/transversal to the path.

\subsubsection{Tangential States} \label{section:tangential}
For the moment assume that the closest spline to the output $y$ is known; denote it as $k^*$. In the case that $\nsplines=1$, the $k$ notation can be dropped to ease readability. Furthermore, denote the parameter of spline $k^*$ that corresponds to the closest point to the output $y$ as $\lambda^* \in \mathbb{I}_{k^*}$ \cite{TFL:hladio2011path}
\begin{equation}
	\label{eqn:lambda_star}
	\lambda^* \coloneqq \varpi_{k^*}(y) \coloneqq \arginf_{\lambda \in \mathbb{I}_{k^*}}\norm{y-\sigma_{k^*}(\lambda)}.
\end{equation}

The first tangential state is the projected, traversed arclength along the path:
\begin{align}
	\label{eq:eta1}
	\eta_1 = \eta_{1_{k^*}}(x) \coloneqq s_{k^*} \circ \varpi_{k^*} \circ h(x) + \sum_{k=1}^{k^*-1}s_k(\lambda_{(k,max)})
\end{align}
where
\begin{align*}
  s_{k}(\lambda) \coloneqq \int\limits_{\lambda_{(k,min)}}^{\lambda}\norm{\frac{\mathrm{d} \sigma_{k^*} (\lambda)}{\mathrm{d} \lambda}} \mathrm{d} \lambda.
\end{align*}
This integral does not have to be computed for real time implementation if only tangential velocity control is required, which will be the case in the experiment (Section \ref{section:Experiment}). Note the slight abuse of notation for using $\eta_1$ to represent both the state and the map. The summation term on the right in \eqref{eq:eta1} is used to get the arclength over the entire path $\mathcal{P}$, so that $\eta_1$ remains continuous even as the current spline number $k^*$ increments/decrements. Note that this summation can be computed offline and accessed as a look-up-table during real-time control.

With $\eta_1$ representing the projected tangential position of the
output $y$ along the path, tangential velocity is computed by taking
the time derivative. By the chain rule,
\begin{align} \label{eq:eta2} \begin{split}
	\eta_2 = \eta_{2_{k^*}}(x) \coloneqq \dot{\eta_1} &=  \frac{\mathrm{d} s_{k^*}}{\mathrm{d} \lambda^*} \frac{\mathrm{d} \varpi_{k^*}}{\mathrm{d} y}  \frac{\partial h}{\partial x_c} \\
			&=  \left.\left\langle \vect{e}_{1_{k^*}}(\lambda^*) , J(x_c)x_v \right\rangle \right|_{\lambda^*=\varpi \circ h(x)}
\end{split}
\end{align}
(see Appendix \ref{appendix:derivation_of_eta2} for complete derivation) where $\vect{e}_1$ is the unit-tangent Frenet-Serret vector:
\begin{equation} \label{eq:e_1} \vect{e}_{1_{k^*}}(\lambda^*)
  \coloneqq \frac{\mathrm{d} \sigma_{k^*}}{\mathrm{d} s_{k^*}} =
  \left. \frac{\mathrm{d} \sigma_{k^*}}{\mathrm{d} \lambda}
  \right|_{\lambda=\lambda^*} \frac{\mathrm{d} \lambda^*}{\mathrm{d}
    s_{k^*}} =
  \frac{\sigma_{k^*}'(\lambda^*)}{\norm{\sigma_{k^*}'(\lambda^*)}}.
\end{equation}. 

\subsubsection{Numerical Optimization for $\lambda^*$ and $k^*$}
\label{section:numerical_opt} 
The tangential coordinates defined
by~\eqref{eq:eta1},~\eqref{eq:eta2} rely on the
computation of the parameters $k^\ast$ -- which characterizes which
of the $\nsplines$ curves $\sigma_k(\mathbb{I}_k)$ is closest to the
system's output -- and $\lambda^\ast$ -- which characterizes the point $\sigma_{k^\ast}(\lambda^*)$ closest to the system's output.

\begin{assumption}[Initial $\lambda^*$ and $k^*$ are known]
\label{assumption:numerical_algorithm_initialization}
Based on the initial conditions of the system at time $t=0$, the corresponding $k^* \in \{1,...,\nsplines\}$ and $\lambda^* \in \mathbb{I}_{k^*}$ are known. 
\end{assumption}

Assumption \ref{assumption:numerical_algorithm_initialization} can be satisfied by \textit{apriori} global minimum brute-force optimization solvers to find the closest point of $h(x(0))$ to the path $\mathcal{P}$. One can do this by quantizing the path $\mathcal{P}$ and searching for the point with minimum distance to the output, whilst keeping track of the curve number $k$ and the path parameter $\lambda$. The output to this approach will not yield the exact solution (due to the quantization of $\mathcal{P}$), but will be adequate for sufficiently small quantization. In the case that multiple global minima exist, one can be chosen arbitrarily. Then for $t>0$, a numerical local optimizer like gradient descent can be used online, initialized at the previous time step's $k^*, \lambda^*$. The algorithm we developed (Algorithm \ref{Algorithm:grad_descent}) can be found in Appendix \ref{appendix:Numeric_opt}.

\begin{remark}	
  In~\cite{TFL:hladio2011path}, the path to be followed was assumed to
  be a one-dimensional embedded submanifold, which precludes
  self-intersecting curves. By employing Algorithm
  \ref{Algorithm:grad_descent} under Assumption
  \ref{assumption:numerical_algorithm_initialization}, this work
  removes that restriction and enlarges the class of admissible paths.

  Consider first the case in which two curves $\sigma_i$, $\sigma_j$,
  $i, j \in \set{1, \ldots, \nsplines}$ overlap. That is, there exist
  two closed, possibly disconnected, sets
  $\Lambda_i \subset \mathbb{I}_i$, $\Lambda_j \subset \mathbb{I}_j$
  with
  $(\forall \lambda_i \in \Lambda_i)(\exists \lambda_j \in \Lambda_j)
  \: \sigma_i(\lambda_i) = \sigma_j(\lambda_j)$.
  Suppose that at the current time-step, we have $k^\ast = i$ and
  $\lambda^\ast = \lambda_i \in \Lambda_i$. Then the coordinate
  transformation employed is completely determined by the curve
  $\sigma_i$. At the next time-step, the algorithm will search
  $\mathbb{I}_k$ for its next update. This local search implies that
  the algorithm will not ``jump'' to $\mathbb{I}_j$ despite the
  existence of intersection points.

    The second case, the one in which there is single curve that is
    not injective, is handled similarly. Suppose there exist two
    values $\lambda_1, \lambda_2 \in \mathbb{I}_{k^\ast}$, $\lambda_1
    < \lambda_2$, with $\sigma_{k^\ast}(\lambda_1) =
    \sigma_{k^\ast}(\lambda_2)$. By continuity, it is possible to find
    two open intervals $\mathbb{I}_1$, $\mathbb{I}_2 \subset
    \mathbb{I}_{k^\ast}$ with $\lambda_1 \in \mathbb{I}_1$ and
    $\lambda_2 \in \mathbb{I}_2$ and $\mathbb{I}_1 \cap \mathbb{I}_2 =
    \varnothing$. Now suppose that at the current time-step, we have
    $\lambda^\ast = \lambda_1$. At the next time-step, the algorithm will
    search $\mathbb{I}_1$ for its next update. This local search
    implies that the algorithm will not ``jump'' to $\mathbb{I}_2$
    despite the existence of an intersection point.

    In addition, when the output $y$ is equidistant to multiple points
    on the path, the algorithm will again just choose the $\lambda^*$
    that is closest to the previous time-step's $\lambda^*$, due to
    the local nature of the monotonic gradient descent
    algorithm. \hfill $\bullet$
\end{remark}

\subsubsection{Transversal States} \label{section:transversal_states}
The transversal states represent the cross track error to the path
using the remaining FSF (normal) vectors as follows. The $k^*$
notation will be dropped for brevity. The generalized Frenet-Serret
(FS) vectors are constructed applying the Gram-Schmidt
Orthonormalization process to the vectors $\sigma^\prime(\lambda)$,
$\sigma^{\prime\prime}(\lambda)$, \ldots, $\sigma^{(p)}(\lambda)$:
\begin{equation} \label{eq:FS_vector}
		\vect{e}_j(\lambda) \coloneqq \frac{ \bar{\vect{e}}_j(\lambda)  }{ \norm{\bar{\vect{e}}_j(\lambda)} }
\end{equation} 
where
\begin{equation*} 
  \quad \bar{\vect{e}}_j(\lambda) \coloneqq \sigma^{(j)}(\lambda) - \sum_{i=1}^{j-1} \left\langle \sigma^{(j)}(\lambda),\vect{e}_i(\lambda) \right\rangle \vect{e}_i(\lambda), 
\end{equation*}
for $j \in \set{1,\ldots,p}$. This formulation is well-defined by
Assumption \ref{assumption:path_framed}. For curves that violate
Assumption \ref{assumption:path_framed}, other ad-hoc
orthonormalization processes can be used, see Remark
\ref{remark:assumption_framed_relax}.
The transversal positions can be computed by projecting the error to the path, $y-\sigma(\lambda^*)$, onto each of the FS normal vectors: 
\begin{equation} \label{eq:transversal_positions} \xi_1^{j-1} =
  \xi_1^{j-1}(x) \coloneqq \left. \left\langle \vect{e}_j(\lambda^*) ,
      h(x)-\sigma(\lambda^*) \right\rangle \right|_{\lambda^*=\varpi
    \circ h(x)}
\end{equation}
for $j \in \set{2,\ldots,p}$. Taking time derivatives yields (see
Appendix \ref{appendix:derivation_of_xi_vel} for complete derivation)
\begin{multline}  \label{eq:transversal_vel}
	\xi_2^{j-1} \coloneqq \xi_2^{j-1}(x) = \dot{\xi}_1^{j-1} =  \frac{\eta_2(x)}{\norm{\sigma'(\lambda^*)}} \left\langle  \vect{e}_j'(\lambda^*), h(x) - \sigma(\lambda^*) \right\rangle 
	\\+ \left\langle  \vect{e}_j(\lambda^*) , J(x_c)x_v \right\rangle |_{\lambda^*=\varpi \circ h(x)},
\end{multline}
$j\in \set{2,\ldots,p}$, and $\vect{e}_j'(\lambda^*)$ are given in \eqref{eq:generalized_FS_equations}. If $p=2$, then $\vect{e}_2'(\lambda^*)$ and
$ h(x) - \sigma(\lambda^*)$ are orthogonal and $\xi_2^1$ simplifies to
$\left\langle \vect{e}_2(s(\lambda^*)) , J(x_c)x_v \right\rangle$.
Figure~\ref{fig:3d_frenet} illustrates the $p=3$ case.
\begin{figure}[h!]
	\centering
	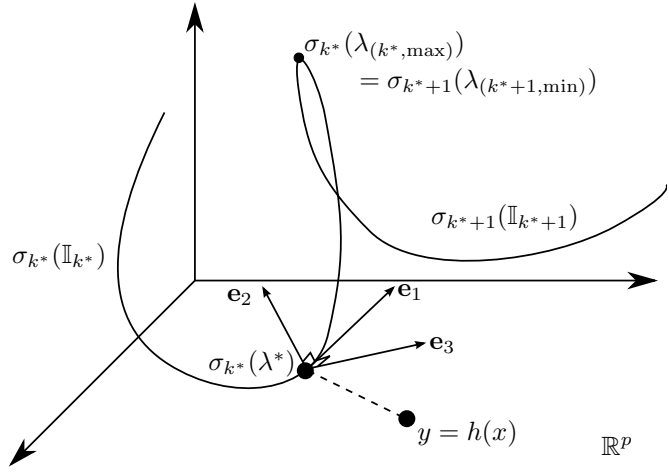
	\caption{Example of the FSF when $p=3$.}
	\label{fig:3d_frenet}
\end{figure}

\subsubsection{Diffeomorphism} \label{section:Diffeomorphism} In this
section, we show that the tangential and transversal maps defined
above can be used to construct a diffeomorphism. For brevity we again
drop the spline $k$ notation. First, define the lift of curve $\sigma$
to $\Real^n$ as
\begin{equation*}
	\Gamma \coloneqq \{ x \in \Real^n : \xi_1^1(x) = \xi_1^2(x) = ... = \xi_1^{p-1}(x) = 0   \}
\end{equation*}
which represents the states $x$ that correspond to the output $y$
being on the curve. The above set is not necessarily a manifold since the curves considered can be self-intersecting or even self-overlapping, thus is more general to the restricted manifold case analysed in \cite{TFL:hladio2011path}.
%

\begin{lemma} \label{lemma:linear_independent_diff} There exists a
  domain $U \subseteq \Real^n$ with $\Gamma \subset U$ such that for
  all $x \in U$, and \eqref{eqn:lambda_star} is solved by Algorithm \ref{Algorithm:grad_descent}, the $2p$
  functions~\eqref{eq:eta1},~\eqref{eq:eta2},~\eqref{eq:transversal_positions},~\eqref{eq:transversal_vel}
  have linearly independent differentials. 
\end{lemma}
\begin{proof}
See Appendix \ref{appendix:proof_lemma_LIDIFF}.
\end{proof}

By \cite[Proposition 5.1.2]{isidori1999nonlinear}, for each $x^* \in U$, it is possible to find a function $\varphi:\Real^n \rightarrow \Real^{n-2p}$ so that by the Inverse Function Theorem \cite{isidori1999nonlinear} there exists a neighbourhood $U_{x^*}$ such that the mapping 
\begin{align} \begin{split} \label{eq:diffeomorphism}
    T :  U_{x^*} \subseteq \Real^n  &\rightarrow T(U_{x^*}) \subseteq \Real^n \\
    x &\mapsto ( \eta_1 , \eta_2 , \xi_1^1 , \xi_2^1, \hdots,
    \xi_1^{p-1} , \xi_2^{p-1} , \zeta) \\ &= (\eta_1(x), \eta_2(x),
    \xi_1^1(x), \xi_2^1(x) , \hdots , \\ &\qquad \qquad \qquad
    \xi_1^{p-1}(x) , \xi_2^{p-1}(x) , \varphi(x))
\end{split}
\end{align}
is a diffeomorphism onto its image. In practice the domain of $T$ can
often be extended to all of $U$ (see Section
\ref{section:Experiment}).

Let us bring back the spline notation $k$ to explicitly indicate that
the coordinate transformation is dependent on the curve $k^*$, which
is determined using Algorithm \ref{Algorithm:grad_descent}. Namely,
the coordinate transformation can be viewed as a collection of $\nsplines$ maps
\begin{align*} 
	T_k :  U_k \subseteq \Real^n & \rightarrow T_k(U_k) \subseteq \Real^n \\
		 x &\mapsto
		 	(\eta_{1_k}(x),
		 	\eta_{2_k}(x),
		 	\xi_{1_k}^1(x),
		 	\xi_{2_k}^1(x),
		 	\hdots, \\ & \qquad \qquad \qquad
		 	\xi_{1_k}^{p-1}(x),
		 	\xi_{2_k}^{p-1}(x),
		 	\varphi_k(x)) = 
		 \end{align*} \begin{equation*} \begin{bmatrix}
		 	s_{k}(\lambda^*) + \sum_{k=1}^{k-1}s_k(\lambda_{(k,max)}) \\
		 	\left\langle \vect{e}_{1_k}(\lambda^*) , J(x_c)x_v \right\rangle \\
		 	\left\langle \vect{e}_{2_k}(\lambda^*) , h(x)-\sigma_k(\lambda^*) \right\rangle \\
		 	\frac{\eta_{2_k}(x)}{\norm{\sigma_k'(\lambda^*)}}\left\langle  \vect{e}_{2_k}'(\lambda^*) , h(x) - \sigma_k(\lambda^*) \right\rangle + \left\langle  \vect{e}_{2_k}(\lambda^*) , J(x_c)x_v \right\rangle \\
		 	\vdots \\
		 	\left\langle \vect{e}_{p_k}(\lambda^*) , h(x)-\sigma_k(\lambda^*) \right\rangle \\
		 	\frac{\eta_{2_k}(x)}{\norm{\sigma_k'(\lambda^*)}}\left\langle  \vect{e}_{p_k}'(\lambda^*) , h(x) - \sigma_k(\lambda^*) \right\rangle + \left\langle  \vect{e}_{p_k}(\lambda^*) , J(x_c)x_v \right\rangle \\
		 	\varphi_k(x)
		 \end{bmatrix}
\end{equation*}
where $\lambda^*=\varpi_k \circ h(x)$, $k \in \{1,\dots,\nsplines\}$.
As we show later, satisfying a path following task depends on a state
feedback involving the $\eta$ and $\xi$ states. Thus, it is desirable
for $\eta$ and $\xi$ to remain continuous from one curve on the path
to the next, i.e., as $k$ changes via Algorithm
\ref{Algorithm:grad_descent}, to avoid wear on equipment due to jerky
motions. The following proposition will show this.
\begin{proposition} \label{proposition:continuous_states}
	The coordinate transformation 
	\begin{equation}
		\bar{T}_k(x) := [\eta_{1_k}(x), \eta_{2_k}(x), \xi_{1_k}^1(x) , \xi_{2_k}^1(x), ..., \xi_{1_k}^{p-1}(x), \xi_{2_k}^{p-1}(x)]^\top
	\end{equation}
	is continuous as $k$ changes by Algorithm \ref{Algorithm:grad_descent}.
\end{proposition}
\begin{proof}
	Based on Algorithm \ref{Algorithm:grad_descent}, $k$ changes when the output of the system crosses the hyperplane shown in Figure \ref{fig:proposition}.
	\begin{figure}[h!]
		\centering
		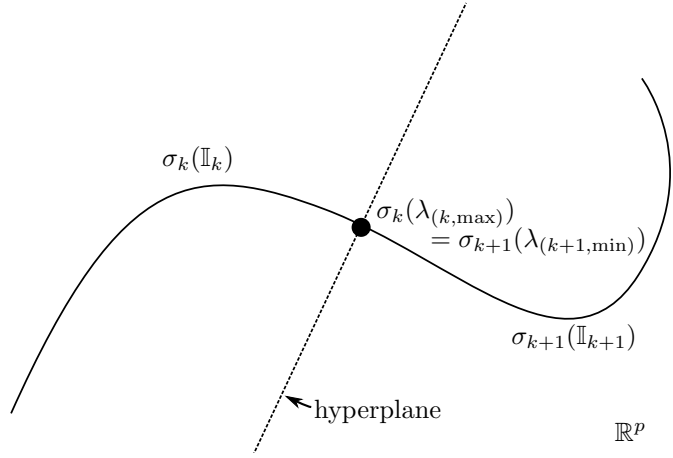
			\caption{Spline $k$ and spline $k+1$, their
                          stitch point $\sigma_{k}(\lambda_{(k,\mathrm{max})}) =\sigma_{k+1}(\lambda_{(k+1,\mathrm{min})})$, and the hyperplane.}
		\label{fig:proposition}
	\end{figure}
	Thus we must show that $\bar{T}_k(x) = \bar{T}_{k+1}(x)$ for $x$ that correspond to the output $y=h(x)$ on this hyperplane. This is equivalent to showing
	\begin{equation*}
		( \forall x \in \{x \in \Real^n | x  = (\varpi_k \circ h(x))^{-1}(\lambda_{(k,\mathrm{max})})\} ) \quad \bar{T}_k(x) = \bar{T}_{k+1}(x).
	\end{equation*}
	First, the coordinate transformation relies on the computation
        of $\lambda^*=\varpi_k \circ h(x)$, which can be easily
        verified to be continuous in $x$ for a fixed $k$. The tangential position for a fixed $k$ is continuous since $s_k(\lambda^*)$ is a
        function of $\sigma'(\lambda^*)$, which is continuous by
        Assumption \ref{assumption:path_smooth}. On the hyperplane,
        $\eta_{1_k}(x) = s_k(\lambda_{(k,max)}) +
        \sum_{i=1}^{k-1}s_i(\lambda_{(k,max)}) =
        \sum_{i=1}^{k}s_i(\lambda_{(k,max)}) = \eta_{1_{k+1}}(x)$
        and is still continuous at the transition. 
	
	For the remaining states, $h(x), J(x)$ are continuous by the
        class of systems \eqref{eq:system_StateSpace}. The remaining
        terms are the FS vectors and its derivatives. For a fixed $k$,
        the FS vectors $\vect{e}_{j_k}(\lambda^*)$ depends on
        $\sigma^{(j)}(\lambda^*)$ for $j=1,...,p$ by
        \eqref{eq:FS_vector}, and thus will be continuous by
        Assumption \ref{assumption:path_smooth}. Furthermore, the
        derivative of the FS vectors are governed by the FS equations
        for general spaces \eqref{eq:generalized_FS_equations}
              , and depend only on the generalized curvatures 
              $\mathcal{X}_i(\lambda) = \frac{\left\langle
                  \vect{e}_{i}'(\lambda), \vect{e}_{i+1}(\lambda)
                \right\rangle} {\norm{\sigma'(\lambda)}}$, $i \in \{1,...,p-1\}$. Thus, the highest
              derivative that appears again is
              $\sigma^{(p)}(\lambda^*)$ and so the FS vector
              derivatives will remain continuous for a fixed spline
              $k$. 
              At the transition point,
              $\vect{e}_{j_k}(\lambda_{(k,\mathrm{max})}) =
              \vect{e}_{j_{k+1}}(\lambda_{(k+1,\mathrm{min})})$
              holds if and only if
              $\sigma^{(j)}_k(\lambda_{(k,\textrm{max})}) =
              \sigma^{(j)}_{k+1}(\lambda_{(k+1,\textrm{min})})$
              for $j \in \set{1,\ldots,p}$. This is satisfied by
              Assumption \ref{assumption:path_smooth}.  To show that
              $\vect{e}_{j_k}'(\lambda_{(k,\mathrm{max})}) =
              \vect{e}_{j_{k+1}}'(\lambda_{(k+1,\mathrm{min})})$
              for $j \in \set{1,\ldots,p}$, 
              by \eqref{eq:generalized_FS_equations} it suffices
              to show that the generalized curvatures are continuous
              at the transition, i.e.,
              $\mathcal{X}_{i_{k}}(\lambda_{(k,\mathrm{max})}) =
              \mathcal{X}_{i_{k+1}}(\lambda_{(k+1,\mathrm{min})})$
              for $i \in \set{1,\ldots,p-1}$.  By definition of the curvatures, this then simplifies to
              showing
              $\vect{e}_{j_k}'(\lambda_{(k,max)}) =
              \vect{e}_{j_{k+1}}'(\lambda_{(k+1,min)})$
              for $j \in \{1,...,p-1\}$. These derivatives can be
              expressed by evaluating the derivative of
              \eqref{eq:FS_vector}, and so the equation will hold if
              $\sigma^{(j)}_k(\lambda_{(k,\textrm{max})}) =
              \sigma^{(j)}_{k+1}(\lambda_{(k+1,\textrm{min})})$
              for $j\in\set{1,\ldots,p}$. Again by Assumption
              \ref{assumption:path_smooth} this holds. 
\end{proof} 

\subsection{Dynamics and Control} \label{sec:dynamics_and_control} The
dynamics in transformed coordinates are
\begin{align} 
	\begin{split} \label{eq:dynamics_unsimplified}
		\dot{\eta}_1 &= \eta_2 \\
		\dot{\eta}_2 &= L^2_f \eta_1(x)  + L_gL_f \eta_1(x) u \\
		\dot{\xi}_1^1 &= \xi_2^1 \\
		\dot{\xi}_2^1 &= L^2_f \xi_1^1(x) + L_gL_f \xi_1^1(x) u \\
		& \vdots \\		
		\dot{\xi}_1^{p-1} &= \xi_2^{p-1} \\
		\dot{\xi}_2^{p-1} &= L^2_f \xi_1^{p-1}(x) + L_gL_f \xi_1^{p-1}(x) u \\
		\dot{\zeta} &= L_f  \varphi(x) + L_g\varphi(x)u =: f_\zeta(x,u)
	\end{split}
\end{align}
where $x=T^{-1}(\eta,\xi,\zeta)$. The Lie derivatives can be computed
given the state maps from Section \ref{section:Diffeomorphism} and the
class of system \eqref{eq:system_StateSpace}:
\begin{align*}
L^2_f \eta_1(x) &= \frac{\eta_2}{\norm{\sigma'(\lambda^*)}} \left\langle \vect{e}_1'(\lambda^*), J(x_c)x_v \right\rangle \\& \qquad + \left\langle \vect{e}_1(\lambda^*) ,  \frac{\partial (J(x_c)x_v)}{\partial x_c}x_v + J(x_c)f_v(x) \right\rangle \\
L_gL_f\eta_1(x) &= \vect{e}_1(\lambda^*)^\top J(x_c)g_v(x_c) \\
L^2_f \xi_1^{j-1}(x) &= \left\langle \vect{e}_j(\lambda^*) ,  \frac{\partial (J(x_c)x_v)}{\partial x_c}x_v + J(x_c)f_v(x) \right\rangle 
\\& \qquad+ \frac{\eta_2}{\norm{\sigma'(\lambda^*)}}\left\langle \vect{e}_j'(\lambda^*) ,2J(x_c)x_v - \vect{e}_1\eta_2 \right\rangle 
\\& \qquad \quad + \left\langle h(x)-\sigma(\lambda^*) , \vect{e}_j''(\lambda^*)\left(\frac{\eta_2}{\norm{\sigma'(\lambda^*)}}\right)^2  \right. \\+ \vect{e}_j'(\lambda^*) & \left.\left( \frac{L_gL_f\eta_1(x)}{\norm{\sigma'(\lambda^*)}}  - \eta_2^2\frac{\inner{\sigma'(\lambda^*)}{\sigma''(\lambda^*)}}{\norm{\sigma'(\lambda^*)}^4} \right) \right\rangle \\
L_gL_f \xi_1^{j-1}(x) &= \vect{e}_j(\lambda^*)^\top J(x_c)g_v(x_c) \\
	& \qquad+(h(x)-\sigma(\lambda^*))^\top\vect{e}_j'(\lambda^*) \frac{L_gL_f\eta_1(x)}{\norm{\sigma'(\lambda^*)}}
\end{align*}
for $j \in \set{2,\ldots,p}$ and $\lambda^*=\varpi \circ
h(x)$.
Physically, see their definitions in Section
\ref{section:coordinate_transform}, the $\eta$-subsystem represents
the dynamics tangent to the path, and the $\xi$-subsystem represents
the dynamics transversal to the path. Thus, the redundant dynamics
$f_\zeta$ represent the dynamics of the system
\eqref{eq:system_StateSpace} that do not produce any output motion.


It then follows from Lemma \ref{lemma:linear_independent_diff} that the decoupling matrix 
\begin{align*}
	&\beta(x) \coloneqq \begin{bmatrix}
		L_gL_f\eta_1(x) \\
		L_gL_f\xi_1^1(x) \\
		\vdots \\
		L_gL_f\xi_1^{p-1}(x)
	\end{bmatrix} = \\
		&\begin{bmatrix}
		\vect{e}_1(\lambda^*)^\top \\
		\frac{\left( h(x)-\sigma(\lambda^*) \right)^\top}{\norm{\sigma'(\lambda^*)}} \vect{e}_2'(\lambda^*)\vect{e}_1(\lambda^*)^\top 
			+ \vect{e}_2(\lambda^*)^\top \\
			\vdots \\
		\frac{\left( h(x)-\sigma(\lambda^*) \right)^\top}{\norm{\sigma'(\lambda^*)}} \vect{e}_p'(\lambda^*)\vect{e}_1(\lambda^*)^\top
			+ \vect{e}_p(\lambda^*)^\top 
	\end{bmatrix} J(x_c) g_v(x_c)
\end{align*}
has full rank for all $x \in U$. This is equivalent to saying that for
each $x \in U$, system \eqref{eq:system_StateSpace} with output
$(\eta_1(x), \xi_1^1(x), \ldots , \xi_1^{p-1}(x))$ has a vector
relative degree of $\set{2,\ldots,2}$ and we can perform partial
feedback linearization. Let
\begin{equation} \label{eq:alpha_x}
\alpha(x) \coloneqq \begin{bmatrix}
L_f^2\eta_1(x) & L_f^2\xi_1^1(x) & \dots & L_f^2\xi_1^{p-1}(x)
\end{bmatrix}^\top
\end{equation} and introduce an auxiliary input $v = \begin{bmatrix}
  v_\eta & v_\xi
\end{bmatrix}^\top \in \Real \times \Real^{p-1}$ such that
\begin{equation} \label{eq:feedback_transform}
	v = \beta(x)u + \alpha(x)
\end{equation}
the dynamics \eqref{eq:dynamics_unsimplified} become
\[
\begin{array}{lll}
\dot{\zeta} = {f}_\zeta(x,u) \qquad & \dot{\eta}_1 = \eta_2   \qquad  &\dot{\xi}_1^j =  \xi_2^j\\
& \dot{\eta_2} = v_\eta    &\dot{\xi}_2^j =v_{\xi^j}\\
\end{array}
\]
where $j \in \set{1, \ldots, p-1}$ and $x=T^{-1}(\eta,\xi,\zeta)$.

In $U$ the $\xi$-subsystem is linear and controllable, and can be
stabilized to ensure attractiveness and invariance of the path
$\mathcal{P}$. One can do this using a simple PD controller
$v_{\xi^j} = -K^j_P\xi_1^j - K^j_D\xi_2^j$,
$j \in \set{1, \ldots, p-1}$, for positive $K^j_P,K^j_D$. Another
transversal controller is \eqref{eq:Robust_outer_control_Khalil} which
is used on the experimental platform to overcome modelling
uncertainties. Doing so will stabilize the largest controlled
invariant subset of $\Gamma$ known as the path following manifold
\cite{TFL:hladio2011path}. The $\eta$-subsystem is also linear and
controllable, thus a tangential controller can be designed for
$v_\eta$ to track some desired tangential position or velocity profile
$\eta^\mathrm{ref}(t)$. A block diagram of the entire control system
can be seen in Figure \ref{fig:block_diagram}.

\begin{remark} \label{remark:eta2ref_trajectory_discontinuous} Suppose
  a desired tangential velocity trajectory is to be tracked,
  $\eta^\mathrm{ref}_2(t)$. In the case that Assumption
  \ref{assumption:path_smooth} fails for $i>0$,
  $\eta^\mathrm{ref}_2(t)$ can be designed such that it approaches
  zero at the junction of the splines, then ramping up to, for
  instance, some desired steady state value just after the
  junction. This way, the control input remains continuous and the
  system avoids unwanted transients associated with the differential
  discontinuities of the path.\hfill $\bullet$
\end{remark}

\subsection{Redundancy Resolution} \label{section:redundancy_resln}
If $n=2p$, there is no $\varphi$ map and thus no redundant dynamics. However, if $n>2p$, then the system has internal dynamics. Unlike kinematic redundancy resolution techniques \cite{spong2006robot}, we seek a scheme at the dynamic level that can incorporate actuator constraints while ensuring the internal dynamics of the system are bounded.  The approach should also work when the joint-space and task-space dynamics are not necessarily related by a static input conversion map (as is the case for torque-input model of a robot manipulator \cite{Khatib1987unified}). Such systems include a combined motor and manipulator model (see Section \ref{section:Experiment}, where a dynamic relation exists between motor input voltage and end-effector force), or mobile manipulator systems. 

Suppose the $\xi$-subsystem is stabilized and the $\eta$-subsystem is tracking a desired profile $\eta^\text{ref}(t)$ by some outer loop controller for $v$. The zero-dynamics correspond to the redundant dynamics $\dot{\zeta} = {f}_\zeta(x,u)$ under the transform \eqref{eq:feedback_transform}, $x=T^{-1}(\eta^\text{ref}(t),0,\zeta)$ and $v_\eta=\dot{\eta}_2^\text{ref}(t),v_\xi=0$. In order to ensure a well-behaved response, we must ensure boundedness of the zero dynamics.

With $v$ generated by some outer loop controller, there is some
freedom in the choice of the control input $u$ under the feedback
transform \eqref{eq:feedback_transform} (recall $v\in\Real^p$,
$u\in\Real^{N=n/2}$ and $n > 2p$); this freedom can be used to enforce
boundedness of the zero dynamics. Since the relation from $u$ to the
states is dynamic, one can use dynamic programming to achieve some
desired (and bounded) trajectory in the zero dynamics. However, this
in general is quite complicated. This is why previous works on path
following have relied on the zero dynamics being absent or assumed to
be stable. To make things simple, consider the static optimization
\begin{align} \label{eq:rendundnacy_optimization}
	\min_u \left( u - r(x) \right)^\top W \left( u - r(x) \right)
\end{align}
under the constraint \eqref{eq:feedback_transform}. It is a static
minimization of a quadratic function of $u$ under a linear constraint,
or which a closed form solution for $u$ can be solved for using
Lagrange Multipliers. The matrix $W \in \Real^{n \times n}$ is an
invertible weighting matrix, and the function
$r : \Real^n \rightarrow \Real^n$ is used to bias the control input
$u$ to achieve desired behaviour in the zero dynamics. For example, if
$r(x) \equiv 0$, then we are minimizing control effort while staying
on the path and following some desired $\eta^\text{ref}(t)$.

Most real systems have actuation and configuration limits. Thus,
another useful criteria for redundancy resolution is to design $u$ so
that the configuration variables $x_c$ are driven away from their
limits. For example, consider a joint on a manipulator. If this joint
is at its minimum value, setting the control effort corresponding to
this joint to be the maximum control effort will \textit{likely}
increase the value of the joint, thereby pushing it away from the
negative joint limit. The corresponding $r$ function to achieve this
behaviour in \eqref{eq:rendundnacy_optimization} is
\begin{align} \label{eq:r_for_optimization}
	r(x_c)_i =  -\frac{ u_{i_{max}}-u_{i_{min}} }{ x_{c_{i_{max}}} - x_{c_{i_{min}}} } (x_{c_i} - x_{c_{i_{min}}}) + u_{i_{max}},
\end{align}
for $i \in \set{1, \ldots, N}$, or graphically shown below in Figure
\ref{fig:r_xc}.
\begin{figure}[h!]
	\centering
	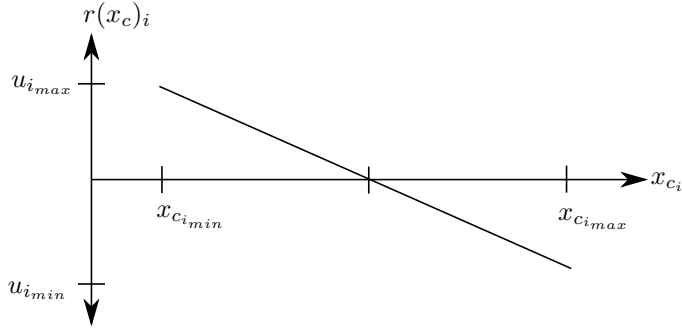
	\caption{The function $r(x_c)$ for avoiding joint and actuation limits.}
	\label{fig:r_xc}
\end{figure}

Using Lagrangian Multipliers, the solution to \eqref{eq:rendundnacy_optimization} under constraint \eqref{eq:feedback_transform} is
\begin{equation} \label{eq:u_function}
	u = \beta^\dagger(x)(v - \alpha(x)) + \left(I_{n \times n} - \beta^\dagger(x)\beta(x) \right)r(x)
\end{equation}
where $\beta^\dagger(x) = W^{-1}\beta(x)^\top \left(\beta(x)W^{-1}\beta(x)^\top\right)^{-1}$ and $I_{n \times n}$ is the identity matrix. 

\begin{conjecture}\label{conjecture:stability_zero_dynamics}
Using \eqref{eq:u_function} and \eqref{eq:r_for_optimization} to generate the control input $u$ results in boundedness of the redundant dynamics $\dot{\zeta} = f_\zeta(x,u)$, while achieving path following and maintaining joint limits, in the situation that system \eqref{eq:euler_lagrange_system} has inherent damping. 
\end{conjecture}

Conjecture \ref{conjecture:stability_zero_dynamics} has been tested and verified on the following two examples \ref{example:stability_1}-\ref{example:stability_2}, and on our experimental platform in Section \ref{section:Experiment}. 

\begin{example}[2-DOF linear system]
\label{example:stability_1}
Consider the linear system in Figure \ref{fig:2dof_figure} with output
$y \in \Real$ taken to be the position of the top-most block. Let
$m_i$, $b_i$ and $u_i$ be the mass, coefficient of friction, and force
acting on each respective block.
\begin{figure}[h!]
	\centering
	\caption{A Planar 2-DOF fully-actuated system where $m_1,m_2,b_1,b_2 > 0$.}
	\label{fig:2dof_figure}
\end{figure}
The corresponding dynamics are 
\begin{align*}
	\dot{x} &= \begin{bmatrix}
		0 & 1 & 0 & 0 \\
		0 & -\frac{b_1+b_2}{m_1} & 0 & \frac{b_2}{m_1} \\
		0 & 0 & 0 & 1 \\
		0 & \frac{b_2}{m_2} & 0 & -\frac{b_2}{m_2} 
	\end{bmatrix}x + \begin{bmatrix}
		0 & 0 \\
		\frac{1}{m_1} & 0 \\
		0 & 0 \\
		0 &  \frac{1}{m_2}
	\end{bmatrix}u \\
	&=: Ax + Bu \\
	y &= \begin{bmatrix}
		0 & 0 & 1 & 0
	\end{bmatrix}x.			
\end{align*}

For this particular example, $N=2,p=1$. Thus $p-1=0$ and there are no transversal dynamics. Therefore, under our path following context, the position of the output along the path is $\eta_1 = x_3$. The remaining function $\varphi : \Real^4 \rightarrow \Real^2$ is chosen to complete a diffeomorphism. For instance
\begin{align*}
	T : \Real^4 &\rightarrow \Real^4 \\
	x &\mapsto \left(
	 \eta_1(x),
	 \eta_2(x),
	 \zeta_1(x),
	 \zeta_2(x)\right) = \left(
	x_3,
	x_4 ,
	x_1,
	x_2\right)
\end{align*}
yields the dynamics in the transformed coordinates
\begin{align*}
 	\dot{\eta}_1 &= \eta_2 \\
 	\dot{\eta}_2 &= \frac{b_2}{m_2}(\zeta_2-\eta_2) + \frac{1}{m_2}u_2 =: v_\eta \\
 	\dot{\zeta}_1 &= \zeta_2 \\
 	\dot{\zeta}_2 &= -\left(  \frac{b_1+b_2}{m_1}  \right)\zeta_2 + \frac{b_2}{m_1}\eta_2 + \frac{1}{m_1}u_1.
\end{align*}
The corresponding terms in the feedback transform \eqref{eq:feedback_transform} are $\alpha(x) = \frac{b_2}{m_2}(x_2-x_4)$ and $\beta(x) = \begin{bmatrix}
0 & \frac{1}{m_2} 
\end{bmatrix}$. Now assume the desired tangential reference trajectory is 
\begin{equation*}
	\eta^\mathrm{ref}(t) = \begin{bmatrix}
		\eta_1^\mathrm{ref} \\
		0
	\end{bmatrix}
\end{equation*}
for some constant $\eta_1^\mathrm{ref}$. The control law $v_\eta = K_P(\eta_1^\mathrm{ref}-\eta_1) + K_D(-\eta_2)$ for positive $K_P, K_D$ will force $\eta$ to converge to $\eta^\mathrm{ref}$. Thus the zero dynamics are the redundant dynamics when $\eta=\eta^\mathrm{ref}$ and $\alpha(x)+\beta(x)u=v_\eta=0$, i.e. 
\begin{align*}
 	\dot{\zeta}_1 &= \zeta_2 \\
 	\dot{\zeta}_2 &=  -\left(  \frac{b_1+b_2}{m_1}  \right)\zeta_2 + \frac{1}{m_1}u_1.
\end{align*}
Using the controller \eqref{eq:u_function} yields the zero dynamics 
\begin{align*}
 	\dot{\zeta}_1 &= \zeta_2 \\
 	\dot{\zeta}_2 &=  -\left(  \frac{b_1+b_2}{m_1}  \right)\zeta_2 + \frac{1}{m_1}r(x)_1.
\end{align*}
and thus the zero dynamics can be made globally exponentially stable if, for example, $r(x)_1$ is a linear function of $\zeta_1=x_1$ with negative slope. This is precisely the case with our proposed $r$ function \eqref{eq:r_for_optimization}. Furthermore, the equilibrium point will be $\zeta_2=0$, and the value of $x=T^{-1}(\eta^\mathrm{ref},\zeta_1,\zeta_2=0)$ for which $r(x)_1$ vanishes. In our proposed $r$ function, this occurs when $\zeta_1=x_1=\frac{u_{1_{max}}}{u_{1_{max}}-u_{1_{min}}}(x_{1_{max}}-x_{1_{min}})+x_{1_{min}}$. If $u_{1_{max}}=-u_{1_{min}}$, which usually is the case for practical examples, then the equilibrium point is  $\zeta_1=\frac{1}{2}(x_{1_{min}}+x_{1_{max}})$, the average value of the joint limits. Thus, the path following constraints are met and the zero dynamics are stabilized while achieving the objective of staying within joint limits using \eqref{eq:r_for_optimization}.
\end{example}

\begin{example}[3-DOF planar manipulator]\label{example:stability_2}
Consider a planar 3-DOF robot manipulator and the desired path of a circle, and the point $\eta_1=0$ on the path is to be stabilized, shown in Figure \ref{fig:3dof_figure}.
\begin{figure}[h!]
	\centering
	\includegraphics[width=\linewidth]{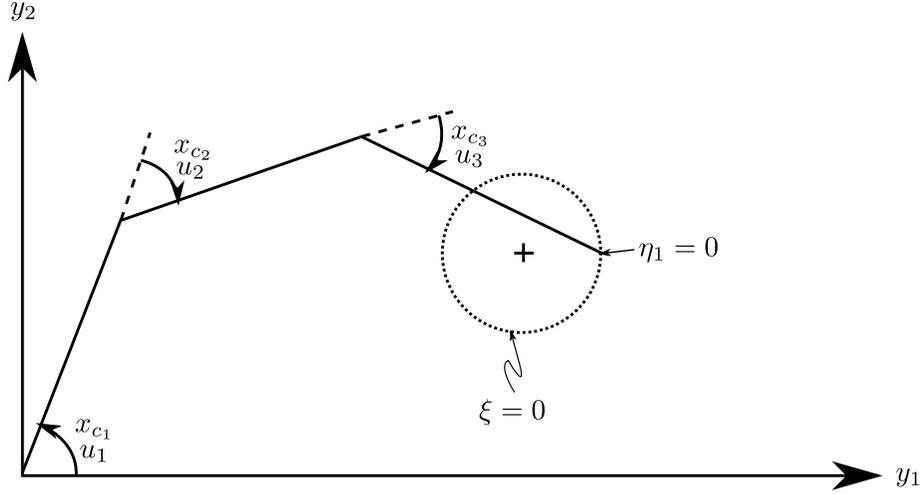}
	\caption{A Planar 3-DOF fully-actuated manipulator whose task is to stabilize a point on a circle.}
	\label{fig:3dof_figure}
\end{figure}

Here $N=3,p=2$ and thus the dimension of redundancy is $1$. Let's
further assume that each link has unit mass and inertia, with the
center of mass of each link be at the midpoint of each link. Without
loss of generality, we assume gravity is ignored. The dynamics can be
derived as per \cite{spong2006robot} to yield the form
\eqref{eq:euler_lagrange_system}.

Let the redundant state be
$\zeta_1 = \zeta_1(x) = x_{c_1} + x_{c_2} + x_{c_3}$, the sum of the
joint angles, the angle the end-effector makes with the ground. Thus
$\zeta_2 = x_{v_1} + x_{v_2} + x_{v_3}$, and
$\dot{\zeta}_2 =
\left.\sum_{i=1}^3\left[f(x)+g(x)u\right]_i\right|_{x=T^{-1}(\eta,\xi,\zeta)}$.
The zero dynamics correspond to these redundant dynamics when
$\xi = 0, \eta=0$ and thus $v=0$.  Stability of the zero dynamics
depends on the choice of $u$.

Suppose that $u$ is chosen as in \eqref{eq:u_function} with $r$ from \eqref{eq:r_for_optimization}. The parameters for the $r$ function are $u_{i_{max}} = -u_{i_{min}} = 10$ for $i=1,2,3$. The joint limits are chosen so that $x=T^{-1}(\eta=\eta^\mathrm{ref},\zeta_1=0,\zeta_2=0)$ is the average value of the joint limits. This is done to see if $\zeta=(0,0)$ will become an equilibrium point for the zero dynamics.

The resulting phase curves for the zero dynamics for the 3-DOF simple robot case with viscous friction can be found in Figure \ref{fig:simple_phasecurve}.
\begin{figure}[h!]
	\centering
	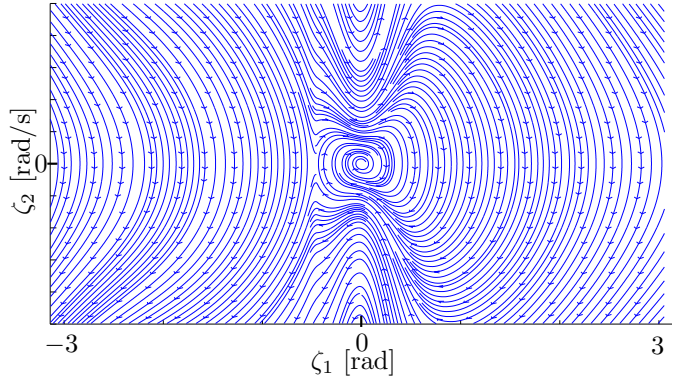
	\caption{Phase portrait of the zero dynamics for the 3-DOF manipulator example with $G=0$.}
	\label{fig:simple_phasecurve}
\end{figure}

Notice there are two equilibrium points in the zero dynamics phase
plane: One is unstable and one that is asymptotically stable. The
unstable equilibrium point is the singularity configuration of the
manipulator on the zero dynamics manifold. The asymptotically stable
equilibrium point is located at $\zeta=(0,0)$, which corresponds to
the joint positions being at the midpoint of their joint
limits. 
Furthermore, if the system were to have no viscous friction,
$\zeta=(0,0)$ will no longer be asymptotically stable. 
\end{example}

\begin{remark}
The proposed approach assumes that the system is fully actuated.  To examine the under-actuated system case, let $m$ be the number of inputs, $p$ be the dimension of the output space, and $N$ be the degrees of freedom of the system. There is a basic requirement (\cite{TFL:hladio2011path}) that $m \geq p-1$ in order to only follow the path, and $m \geq p$ in order to follow the path and move along it in a desired fashion. 

If $N > m = p$, it is still possible to follow the path in a desired fashion, however the internal dynamics would be completey uncontrollable, since all the freedom over the control input will be used to stay on the path and traverse the path. If $N = m > p$, then Conjecture \ref{conjecture:stability_zero_dynamics} applies. If $N > m > p$, then there will be some uncontrolled internal dynamics and the proposed redundancy resolution approach may not work to ensure boundedness of these internal dynamics. \hfill $\bullet$
\end{remark}
\begin{remark}
In the case that the system hits a singularity ($J(x)$ loses rank), the controller will fail as the decoupling matrix $\beta$ loses rank, and there will no longer be a solution to the feedback transformation \eqref{eq:feedback_transform} or \eqref{eq:u_function}. The proposed methodology assumes that $J$ is non-singular by the class of systems (Section \ref{section:class_of_systems}). Practically, one just has to ensure $J(x)$ is non-singular in the neighbourhood of $\Gamma=h^{-1}(\mathcal{P})$. \hfill $\bullet$
\end{remark}

\subsection{Continuity of control input}
The tangential and transversal dynamics using the linearizing control input $v$ become a series of double integrators. Stabilizing such dynamics can be done using linear control techniques. For the path $\mathcal{P}$ it would be nice for the control input $u$ to remain continuous as the output moves from one curve to another within $\mathcal{P}$. 
\begin{theorem}
	If the auxiliary input $v$ is chosen to be a continuous function of $\bar{T}_k(x)$ with $k=k^*$ selected by Algorithm \ref{Algorithm:grad_descent}, then the control input $u$ is continuous for the entire splined path $\mathcal{P}$. 
\end{theorem}
\begin{proof}
	If $v$ is a continuous function of $\eta$ and $\xi$, then $v$ is continuous if and only if the $\eta$ and $\xi$ states are continuous. This is in fact the case by Proposition \ref{proposition:continuous_states}, thus $v$ is continuous over the entire path. 
	
	The control input $u$ is calculated based on the feedback transform \eqref{eq:feedback_transform} for the $k$'th spline:
	\begin{equation*}
		v = \beta_{k^*}(x)u + \alpha_{k^*}(x)
	\end{equation*}
	
	Since $v$ is continuous, $u$ is continuous if and only if $\beta_{k^*}(x)$ and $\alpha_{k^*}(x)$ are continuous (assuming the optimization to solve for $u$ in the redundant case is also smooth). Equivalently, in a similar manner to Proposition \ref{proposition:continuous_states}, we must show that 
	\begin{itemize}
			\item $\beta_{k}(x)$ and $\alpha_{k}(x)$ are continuous given $k$; 
			\item $\beta_{k}(x)$ and $\alpha_{k}(x)$ are continuous at the transistion of $k$: \begin{multline*}
		( \forall x \in \{x \in \Real^n | x  = (\varpi_k \circ h(x))^{-1}(\lambda_{(k,max)})\} ) \\ \beta_{k}(x)= \beta_{k+1}(x) \;,\; \alpha_{k}(x) =  \alpha_{k+1}(x)
		\end{multline*}
	\end{itemize}
	
	Using the definition for $\beta_k(x)$ from Section \ref{sec:dynamics_and_control},
	the term $J(x_c)g_v(x_c)$ is continuous over the path (by system assumptions) and is not dependent on the spline $k$. The remaining terms are all used in $\bar{T}_k(x)$, and thus it follows from Proposition \ref{proposition:continuous_states} that $\beta_k(x)$ is continuous.  
	
	
	Similarly, using the definition for $\alpha_k(x)$ from Section \ref{sec:dynamics_and_control}, the terms 
        $\vect{e}_{j_k}''(\lambda^*)$, $j \in \set{1,\ldots ,p}$ are
        continuous for a given spline $k$ because by using the FS
        formulae from Proposition \ref{proposition:continuous_states},
        the highest derivative that appears will be
        $\sigma_k^{(p+1)}(\lambda^*)$; our path is $C^{p+1}$ by
        Assumption \ref{assumption:path_smooth}. The remaining terms
        in $\alpha$ are continuous since they are the same ones in
        $\bar{T}_k(x)$, which were already shown to be continuous in
        Proposition \ref{proposition:continuous_states}.

\end{proof}


\section{Experiment} \label{section:Experiment}

\subsection{System}
The test platform is the Clearpath Robotic Manipulator (CPM) (see Figure \ref{fig:CPM_picture}) 
\begin{figure}[h!]
	\centering
	\includegraphics[width=\linewidth]{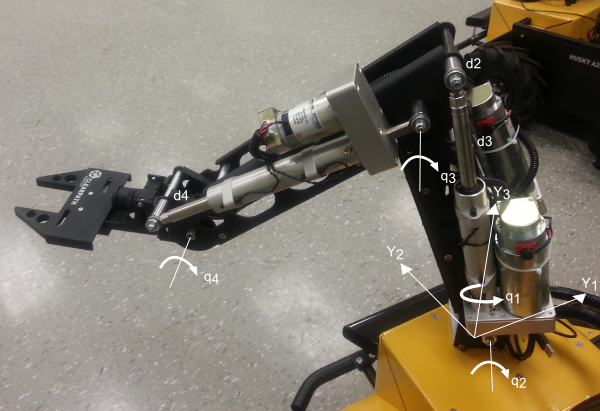}
	\caption{The Clearpath Manipulator (CPM) with the joint coordinates and output space labelled.}
	\label{fig:CPM_picture}
\end{figure}
which is a four degree of freedom, fully actuated system ($N=4$). The shoulder, elbow, and wrist links are actuated by DC linear actuators. The dynamics are derived in \cite{gill2013robust} and is of the form \eqref{eq:euler_lagrange_system}. The dimension of the output space is $p=3$ -- $y_1$ and $y_2$ are the end-effector positions projected on the ground plane, and $y_3$ is the end-effector's height off the ground. The input to the system is the motor voltages. Since $N > p$, the CPM system is redundant.

\subsection{Path Generation}
Waypoints in the 3-dimensional workspace were specified \textit{a-priori}, and a standard spline-fitting process using polynomials \cite{Erkorkmaz2001trajandspline} was used to fit splines through these way points while ensuring fourth-derivative continuity. Thus, quintic polynomial splines are required to fit the waypoints. 

A total of 15 way points are specified. The way points were placed to create self-intersections in the path to show the robustness of Algorithm \ref{Algorithm:grad_descent}. Mathematically, our path is regular and $C^4$, satisfies Assumptions \ref{assumption:path_smooth} and \ref{assumption:path_framed}, and is given as the following after spline interpolation
\begin{align*}
	\sigma_k : \mathbb{I}_k &\rightarrow \Real^3, \qquad k = 1,2,...,14 \\
		\lambda &\mapsto \begin{bmatrix}
							  \sum_{i=0}^{5} a_i \lambda^i \\
							  \sum_{i=0}^{5} b_i \lambda^i \\
							  \sum_{i=0}^{5} c_i \lambda^i
						  \end{bmatrix},
\end{align*}
where $\mathbb{I}_k = [0,l_k]$, and $l_k$ is the chord length between way point $k$ and $k+1$. 

\subsection{Control Design}
The first step in the control design is to define the coordinate transformation. The $k^*$ and $\lambda^*$ are determined by Algorithm \ref{Algorithm:grad_descent}. The initial guess for Algorithm \ref{Algorithm:grad_descent} is determined a-priori by solving the optimization globally using brute force and given the (known) initial output position $y_0 = h(x_0)$. Then, when the real time control law is run, Algorithm \ref{Algorithm:grad_descent} is used to determine the $k^*$ and $\lambda^*$ at each time step. The coordinate transformation follows from Section \ref{section:Diffeomorphism}. Since $n-2p=2$, we need a map $\varphi_k: \Real^n \rightarrow \Real^2$ to complete the diffeomorphism. We let $\varphi_{1_k}(x) = x_{c_2} + x_{c_3} + x_{c_4}$ and $\varphi_{2_k}(x) = x_{v_2} + x_{v_3} + x_{v_4}$ for all $k$ to represent the angle and its rate, respectively, of the end-effector with respect to the ground plane. It can be shown that this completes the diffeomorphism.

The dynamics in the transformed space are then 
\begin{align*}
	\begin{split}
		\dot{\eta}_1 &= \eta_2 \\
		\dot{\eta}_2 &= L^2_f \eta_{1_k^*}(x)  + L_gL_f \eta_{1_k^*}(x) u \\
		\dot{\xi}_1^1 &= \xi_2^1 \\
		\dot{\xi}_2^1 &= L^2_f \xi_{1_k^*}^1(x) + L_gL_f \xi_{1_k^*}^1(x) u \\	
		\dot{\xi}_1^2 &= \xi_{1_k^*}^2 \\
		\dot{\xi}_2^2 &= L^2_f \xi_{1_k^*}^2(x) + L_gL_f \xi_{1_k^*}^2(x) u \\
		\dot{\zeta}_1 &= \zeta_2 \\
		\dot{\zeta}_2 &= \sum_{i=2}^4 [f_v(x)+g_v(x)u]_i	
	\end{split}
\end{align*}
where $x=T_{k^*}^{-1}(\eta,\xi,\zeta)$, $k^*$ is chosen by Algorithm \ref{Algorithm:grad_descent}, and $\dot{\zeta}$ represent the redundant dynamics for path following. 

Setting $v=\beta_{k^*}(x)u + \alpha_{k^*}(x)$ as previously defined in \eqref{eq:feedback_transform}, the following partially linearised system is obtained:
\begin{align*}
		\dot{\eta_1} &= \eta_2        &\dot{\xi_1^1} &= \xi_2^1       &\dot{\xi_1^2} &= \xi_2^2\\
		\dot{\eta_2} &= v_\eta    	  &\dot{\xi_2^1} &=v_{\xi^1}      &\dot{\xi_2^2} &= v_{\xi^2}\\
		\dot{\zeta} &= \bar{f}_\zeta(x,v)			
\end{align*}
where $x=T_{k^*}^{-1}(\eta,\xi,\zeta)$ and $\bar{f}_\zeta$ represents the redundant dynamics based on our choice of control input $u$. The desired task is to track a constant tangential velocity profile (i.e., to track some $\eta_2^\text{ref}$) while staying on the path (force $\xi$ to go to zero). Designing a controller for $v$ is a linear control design problem; the one found in \cite{gill2013robust} is used as it was shown to be robust to the modelling inaccuracies. The tangential controller is a PI controller
\begin{align} 
	\begin{split} \label{eq:tangential_PI_controller_v}
		v_\eta(\eta,t) = K_{P} & \left( \eta_2^{\mathrm{ref}}  - \eta_2(t) \right) + \\ & K_{I} \int\limits_0^t \left(  \eta_2^{\mathrm{ref}} - \eta_2(\tau) \right) \mathrm{d}\tau.
	\end{split} 
\end{align}
For the transversal controller, let
$\xi = \left(\xi_1^1,\xi_2^1,\xi_1^2,\xi_2^2\right)$,
$v_\xi = \left(v_{\xi_1},v_{\xi_2}\right)$. Then,
\begin{equation} \label{eq:Robust_outer_control_Khalil}
	v_\xi = \left(K + K_0\right)\xi + \begin{cases}
	  K_1\xi/\|\xi\| & : \|\xi\| \geq \mu > 0 \\
	  K_2\|\xi\|\xi & : \|\xi\| < \mu.
	\end{cases} 
\end{equation}
where $K, K_0,K_1,K_2 \in \mathbb{R}^{3 \times 3}, \mu \in \mathbb{R}$ are constants that depend on a robustness criteria and are used as tuning parameters \cite{gill2013robust,Khalil:esfandiari1992output}. 

To solve for $u$ based on the auxiliary control $v$, the optimization is done as outlined in Section \ref{section:redundancy_resln} using \eqref{eq:r_for_optimization}. The actuation limits are $u_{i_{max}} = -u_{i_{min}} = 50\%$ duty cycle for $i=1,\dots,4$. The joint limit values $x_{c_{i_{max}}},x_{c_{i_{min}}}$ were set to their true limits for joints $i=1,2,3$, and for the wrist joint ($i=4$) two different runs were done, one with limits set to $0-45$ degrees and another to $45-90$ degrees. The results are found in the subsequent section.

We implement the path following controller on the CPM via Labview Real-Time Module$^\copyright$ to control the motor PWM amplifiers and to read the optical encoders. The loop rate of the controller is 20 ms. The encoders read the linear actuator distances,  so an inverse measurement model is used to retrieve the states $x_c$ of the system. The encoder resolution for the waist is $1.3E-5$ rad, and for the linear actuators is $4.2E-5$ inch. The derivative of $x_c$ is numerically computed to retrieve the $x_v$ states using a first-order approximation. The tangential controller \eqref{eq:tangential_PI_controller_v} is the only dynamic part of the control algorithm, and is discretized also using a first-order approximation. 

\subsection{Results}
The results of the path following controllers using both joint limit values can be found in Figures \ref{fig:3d}-\ref{fig:control_effort}. The output of Algorithm \ref{Algorithm:grad_descent} is also shown for reference, with the number of iterations the gradient descent takes to converge at each time step shown. 

The path following controller naturally converges to the closest point on the path due to the explicit stabilization of the path following manifold. 
\begin{figure}[h!]
	\centering
	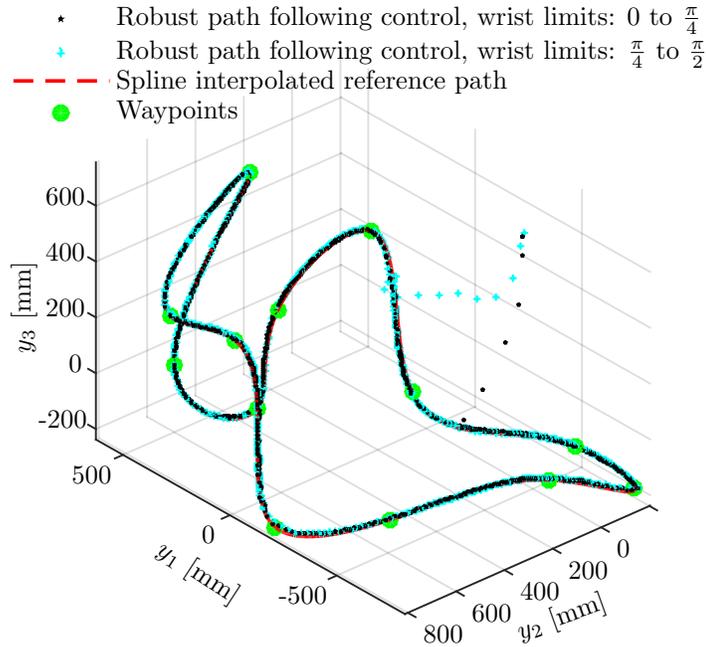
	\caption{Output position.}
	\label{fig:3d}
\end{figure}

\begin{figure}[h!]
	\centering
	\includegraphics[width=\linewidth]{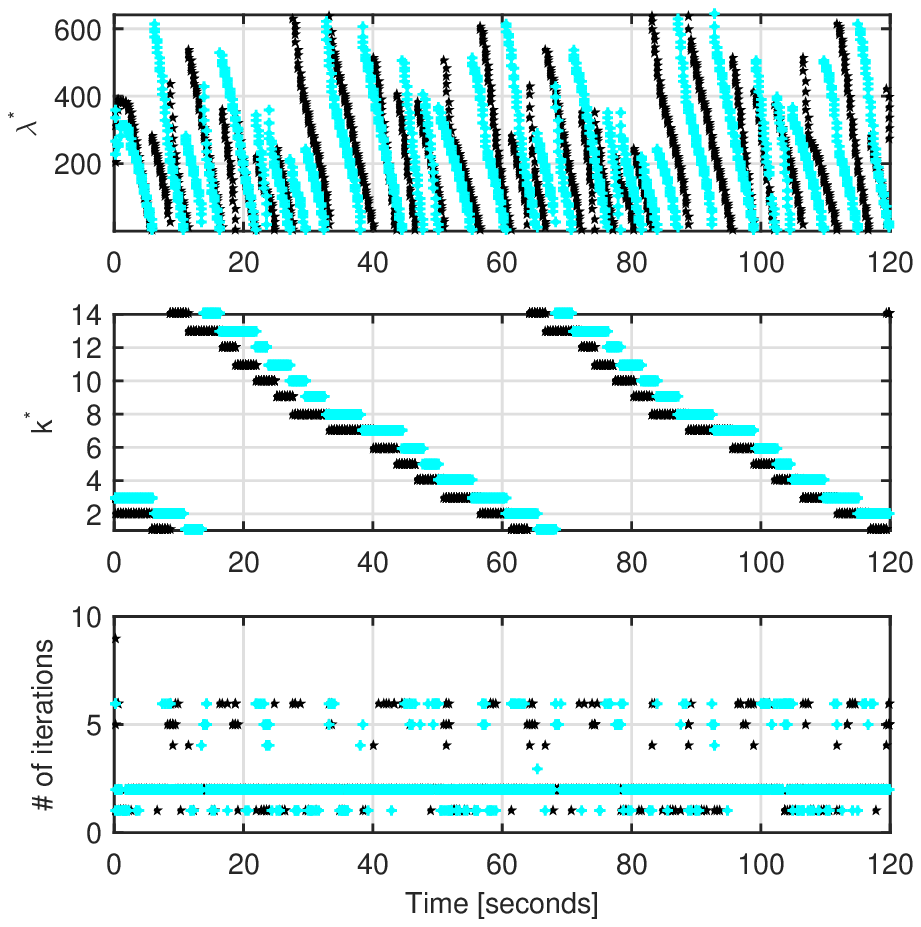}
	\caption{Algorithm \ref{Algorithm:grad_descent} information.}
	\label{fig:aglorithm1}
\end{figure}

Note that in the Cartesian plot, the responses are very similar, with
similar path errors. The transformed state plots also show this as
expected: the controllable $\eta$ and $\xi$ subsystems behave the same
in both scenarios, because these subsystems, under the feedback
transform, do not depend on the redundant state dynamics
($\zeta$). The redundant state $\zeta$ trajectory differs, however,
since we are using the redundancy to maintain different joint angle
limits using
\eqref{eq:rendundnacy_optimization},\eqref{eq:r_for_optimization}. The
plots show well-behaved dynamics and boundedness for $\zeta$ as we had
postulated for systems with inherent damping.
\begin{figure}[h!]
	\centering
	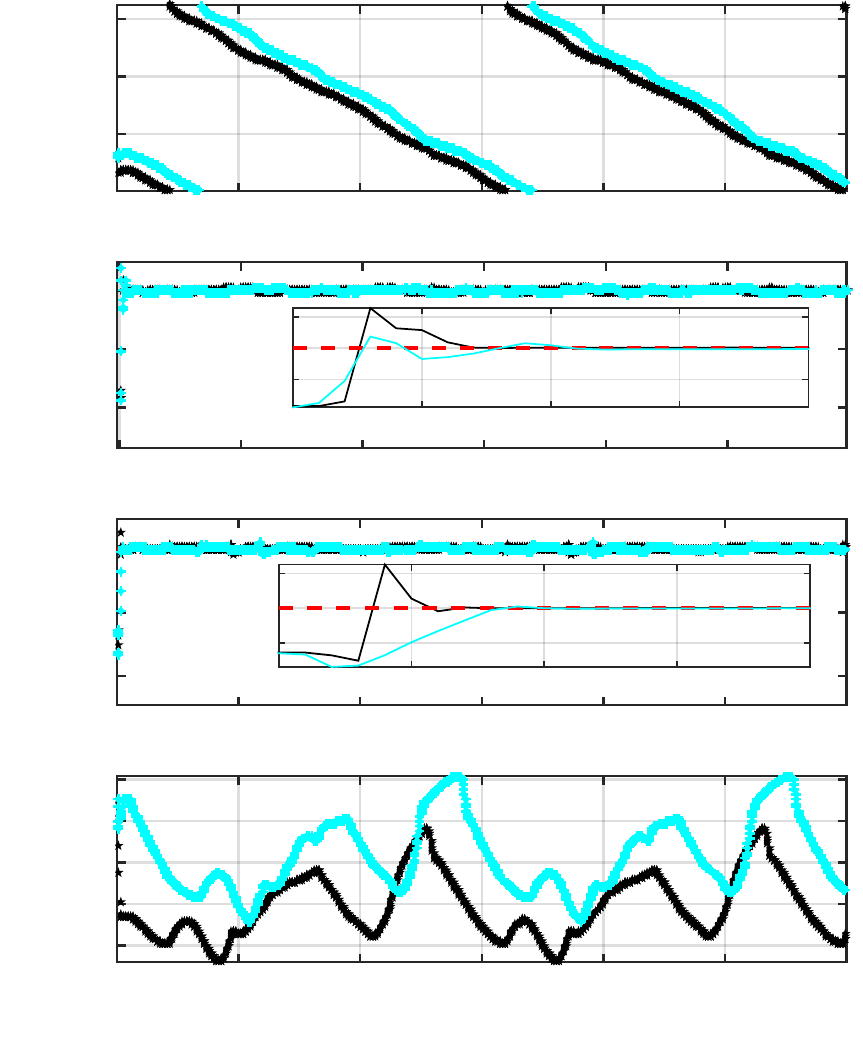
	\caption{Transformed state positions.}
	\label{fig:transformed_pos}
\end{figure}

Closer inspection of the $x_c$ states in Figure \ref{fig:joint_angles} shows that indeed our controller was able to satisfy the respective joint angle limits. The angle trajectories are also slightly different for the first three joints, because maintaining the output on the path with a different preferred position of the wrist forces the angles for the elbow and shoulder links to adjust accordingly. All this is done automatically by \eqref{eq:rendundnacy_optimization},\eqref{eq:r_for_optimization}. 
\begin{figure}[h!]
	\centering
	\includegraphics[width=\linewidth]{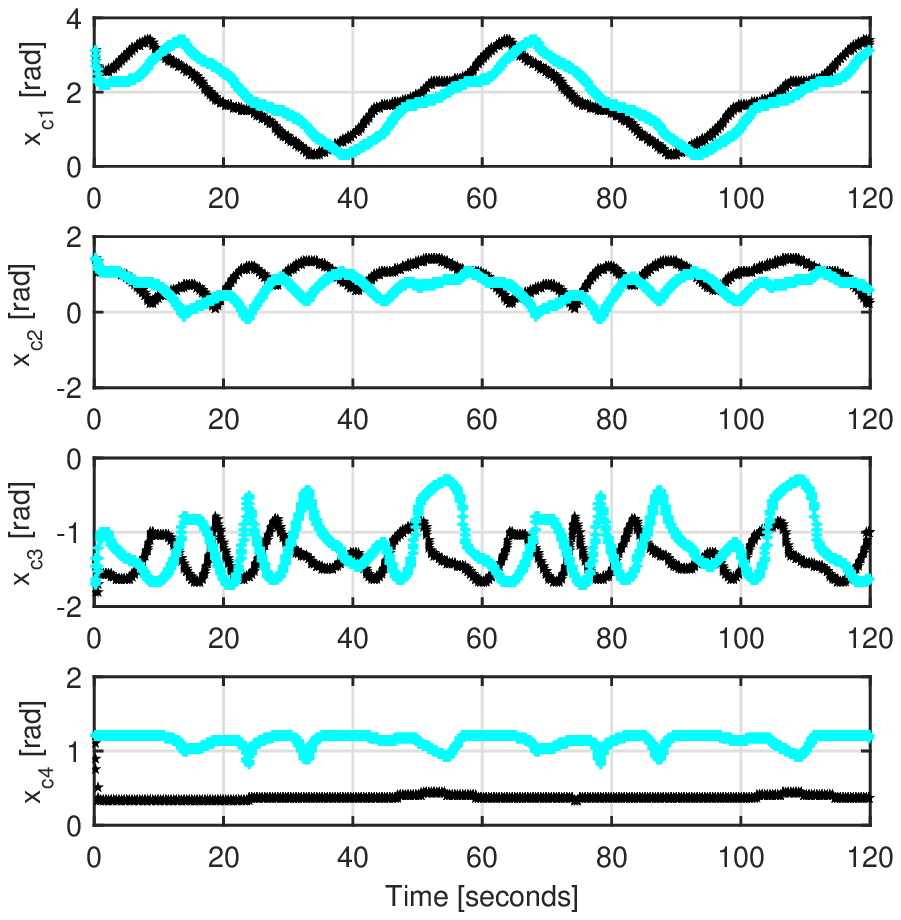}
	\caption{Joint angles.}
	\label{fig:joint_angles}
\end{figure}

Figure \ref{fig:control_effort} shows well behaved control action without chattering. 
\begin{figure}[h!]
	\centering
	\includegraphics[width=\linewidth]{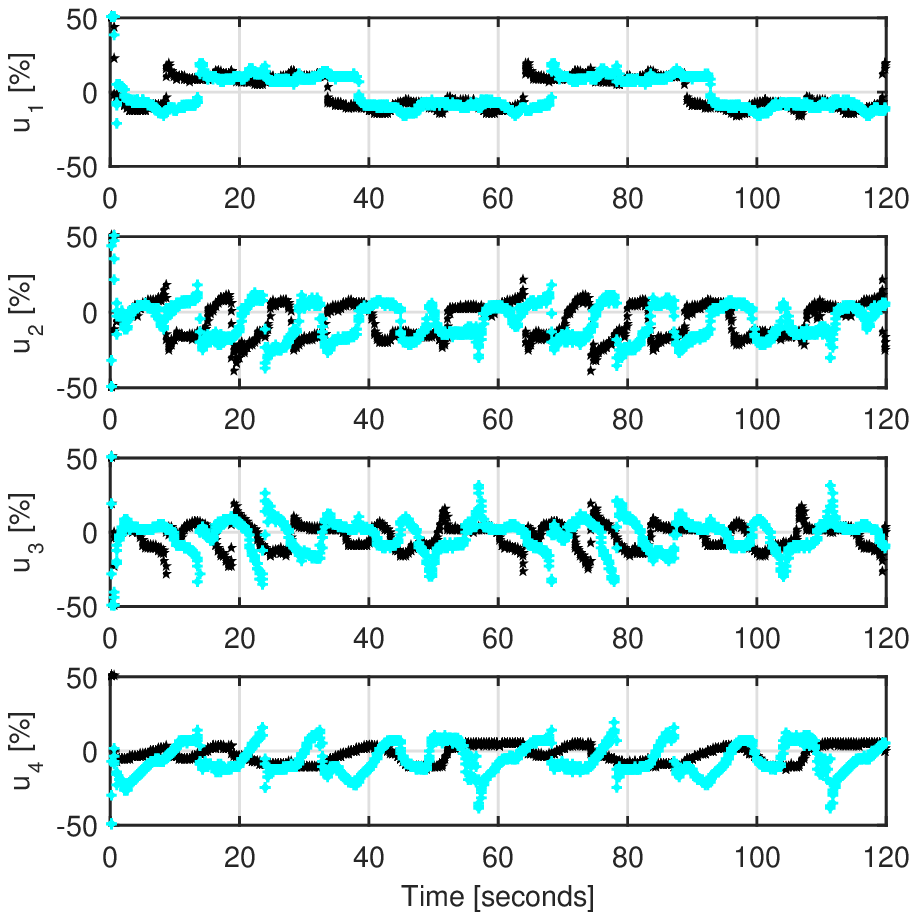}
	\caption{Control effort.}
	\label{fig:control_effort}
\end{figure}

\section{Conclusion and Future Work}
This paper proposes a method for following general paths in the form of sequences of curves; most notably are spline interpolated paths which can be made to go through arbitrary way points. The path following manifold corresponding to this path is stabilized, thus rendering the path attractive and invariant. A numerical algorithm is used to allow the use of paths which are self-intersecting. 

Furthermore, a redundancy resolution scheme is proposed based on a static optimization. A conjecture is made stating that this scheme can yield bounded zero-dynamics and, in particular, satisfy an objective of staying away from joint limits. This scheme has been validated in two analytical examples, as well as on our experimental platform.  A rigorous proof may be direction for future work. Also, a dynamic optimization may be done to account for the system dynamics in the case that the system does not have inherent damping, which may also be direction for future work.


\newpage
\begin{appendices}
\numberwithin{equation}{section}
\begingroup
\footnotesize 
\section{Numerical Optimization} \label{appendix:Numeric_opt} 
Due to the nature of numerical algorithms, the computation of $\lambda^*$ and $k^*$ for $t>0$ is done discretely. At the next time-step $t=\Delta t$ (where $\Delta t$ is sampling time of the computer) and as the output $y$ moves, the next, closest local minimum of $\norm{y-\sigma_{k^*}(\lambda)}$
is the global minimum, assuming that the output $y$ hasn't moved
far within the time step. Thus, simple numerical algorithms like
gradient descent can be used to find the nearest local minimum which
will end up being the global minimum. The idea is illustrated in 
Figure \ref{fig:numerical_optimization} below.
\begin{figure}[h!]
	\centering
	\includegraphics[width=\linewidth]{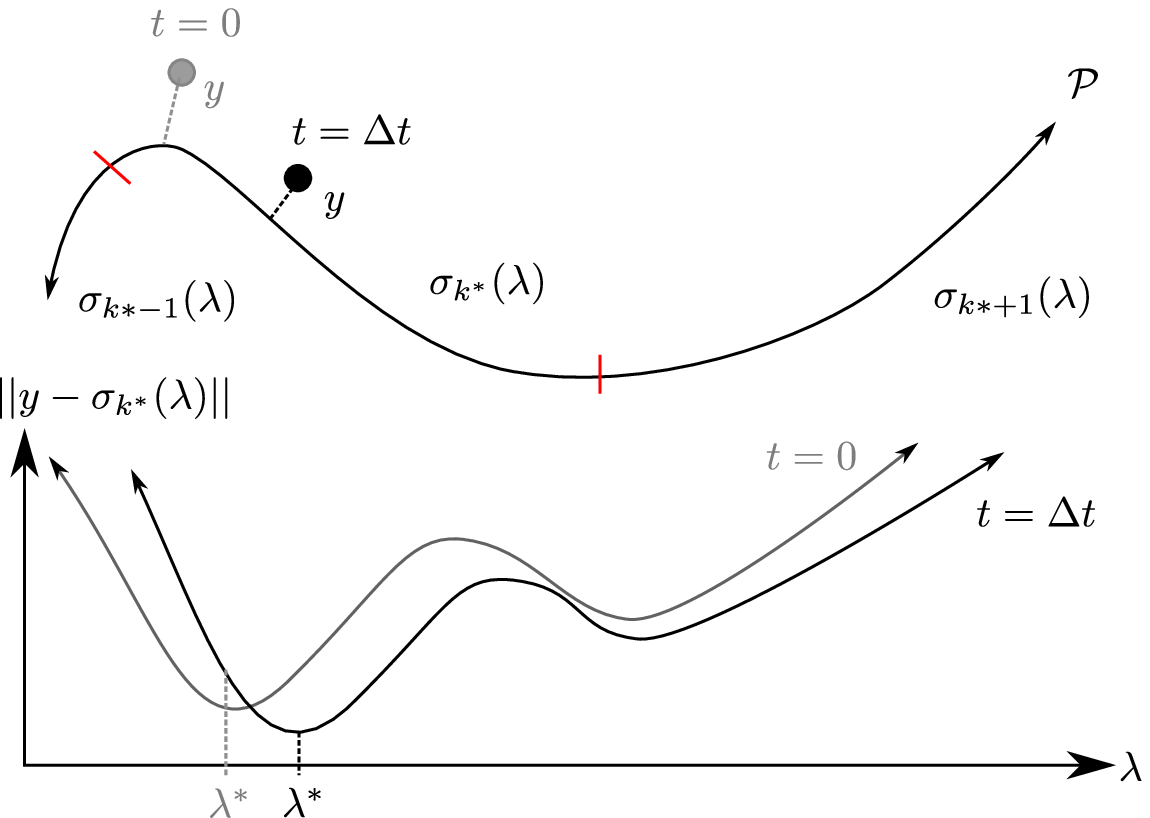}
	\caption{\footnotesize At the start of the run, the true $\lambda^*$ is known. As the output $y$ moves, at each subsequent time-step a numerical optimizer is likely to find the global minimum of $\norm{y-\sigma_{k^*}(\lambda)}$ if initialized at the previous time-step's $\lambda^*$, since the global minimum will be near. 
	}
	\label{fig:numerical_optimization}
\end{figure}

For a moment assume that the output $y$ is already on the path ($y=\sigma_{k^*}(\lambda^*)$). If the output $y$ moves too far within one time step (corresponds to a large change in the path parameter $\lambda^*$), then there may be a local minimum between the initial guess (the previous time-step's solution) and the true global minimum. Thus, we can ensure that a numerical algorithm will not get stuck at such a well by ensuring the path parameter only changes by some amount $\Delta\lambda$, which keeps the function $\norm{y-\sigma_{k^*}(\lambda)}_{y=\sigma_{k^*}(\lambda^*)}$ convex for $\lambda \in [\lambda^* - \Delta\lambda, \lambda^* + \Delta\lambda]$. If $\sigma_{k^*}$ is unit-speed parametrized, i.e., $\norm{\sigma'(\lambda)} = 1 $ for all $\lambda$, then $\Delta\lambda$ corresponds to the allowable distance along the curve that can be travelled. 

This $\Delta\lambda$ is a conservative estimate for the allowable change in parameter in one time-step, since requiring the function to be convex is sufficient, not necessary to ensure convergence of a numerical algorithm to $\lambda^*$ if initialized properly. In logic notation, $\Delta\lambda$ must satisfy
\begin{multline*}
	(\forall \lambda^* \in \mathbb{I}_{k^*})\left(\forall \lambda \in \left[\lambda^* - \Delta\lambda, \lambda^* + \Delta\lambda\right] \cap \mathbb{I}_{k^*} \right) 	\\		\frac{\mathrm{d}^2 \left(\norm{y-\sigma_{k^*}\lambda)}_{y=\sigma_{k^*}(\lambda^*)}\right)}{\mathrm{d}\lambda^2} > 0	
\end{multline*}
which is equivalent to
\begin{multline} \label{eq:critera_path_opt}  \tag{A.1}
	(\forall \lambda^* \in \mathbb{I}_{k^*})\left(\forall \lambda \in \left[\lambda^* - \Delta\lambda, \lambda^* + \Delta\lambda\right] \cap \mathbb{I}_{k^*} \right)  \\
	\left\langle \sigma_{k^*}(\lambda^*)-\sigma_{k^*}(\lambda), \sigma_{k^*}''(\lambda) \right\rangle +\\ \frac{\left\langle \sigma_{k^*}(\lambda^*)-\sigma_{k^*}(\lambda), \sigma_{k^*}'(\lambda) \right\rangle^2}{\norm{\sigma_{k^*}(\lambda^*)-\sigma_{k^*}(\lambda)}^2} < \norm{\sigma_{k^*}'(\lambda)}^2.
\end{multline}

The relation in \eqref{eq:critera_path_opt} can be used as a test to infer the allowable speed at which the output $y$ can traverse a path in order for a numerical optimizer to succeed in evaluating $\varpi_{k^*}(y) = \arginf_{\lambda \in \mathbb{I}_{k^*}}\norm{y-\sigma_{k^*}(\lambda)}$. In particular for a unit-speed parametrized curve and assuming the output is on the path $\mathcal{P}$, if \eqref{eq:critera_path_opt} holds for some $\Delta\lambda$, then the allowable speed along the curve ($\eta_2$) is $\frac{\Delta\lambda}{\Delta t}$. 

\begin{example}[Ellipse] 
  Let $\nsplines=1$ and consider the path of a single ellipse (the
  $k^*$ subscript will be dropped)
  $\sigma(\lambda) = \left(2 \cos(\lambda), \sin(\lambda)\right)$.
  When $\lambda^* = 0$, all $\lambda \in [-1.5136,1.5136]$ satisfies
  the inequality in \eqref{eq:critera_path_opt}. This means that when
  the output is at $\sigma(0)$, the change in the path parameter in
  one time step is $+/- 1.5136$. This value can be solved for
  analytically or numerically using the inequality in
  \eqref{eq:critera_path_opt}. We can run the same test for all
  $\lambda^*$ over the path and take the smallest resulting range to
  be the $\Delta\lambda$.
\end{example}

A numerical algorithm like steepest descent works for finding
$\lambda^*$ if it starts within $\Delta\lambda$ of $\lambda^*$,
however it normally takes many iterations to converge since the
algorithm takes small step sizes to the solution
\cite{bonnans2006numerical}. Another approach is to use some knowledge
about $\eta_2^{\mathrm{ref}}$ to adjust the step size. In particular,
if the path is unit-speed parametrized, one could use an initial step
size of $\eta_2^{\mathrm{ref}}\Delta t$ in the direction of the
steepest descent to quickly approach the solution. If the path is not
unit-speed parametrized, a map of $\eta_1^{\mathrm{ref}}(t)$ and
$\lambda^*$ can be constructed using \eqref{eq:eta1} sampled at each
time step, and the smallest jump in $\lambda^*$ can be used as the
initial step size. An algorithm like monotonic gradient descent with
adaptive step-size works nicely (Algorithm
\ref{Algorithm:grad_descent}) using this initial step size
\cite{marc2012lecture}.

Determining $k^*$ is just a matter of checking if the $\lambda^*$
computed by the numeric algorithm is outside the domain
$\mathbb{I}_{k^*}$. If it is, then $k^*$ must be incremented or
decremented, as done in Algorithm \ref{Algorithm:grad_descent}. Let
$\bar{\varpi}_{k^*}(\lambda) \coloneqq \norm{y-\sigma_{k^*}(\lambda)}$
be the function to be minimized. The algorithm for determining
$\lambda^*$ and $k^*$ can be found in Algorithm
\ref{Algorithm:grad_descent}.
\begin{algorithm} \footnotesize
	\caption{\footnotesize Determining $\lambda^*$ and $k^*$ at each time step.}
	\begin{algorithmic}[1] \label{Algorithm:grad_descent} 
		\REQUIRE The closest spline $k^*$ and corresponding $\lambda^*$, stepsize $\alpha$, the gradient $ \frac{\partial \bar{\varpi}_{k^*}}{\partial  \lambda}(\lambda^*)$, and an $\epsilon$ for the stopping criterion. 
		\ENSURE The closest spline $k^*$ and the $\lambda^*$ minimizing $\bar{\varpi}_{k^*}(\lambda)=\norm{y-\sigma_{k^*}(\lambda)}$.
		\STATE Initialize $\bar{\varpi}_{\lambda^*} = \bar{\varpi}_{k^*}(\lambda^*), \quad g = {\frac{\partial \bar{\varpi}_{k^*}}{\partial  \lambda}(\lambda^*)}^\top$
		\REPEAT 
		\STATE $\lambda^{*\prime} \leftarrow \lambda^* - \alpha g/\norm{g}, \qquad  \bar{\varpi}_{\lambda^{*\prime}} \leftarrow \bar{\varpi}_{k^*}(\lambda^*)$
		\IF{ $\bar{\varpi}_{\lambda^{*\prime}} < \bar{\varpi}_{\lambda^{*}}  $  }
		\STATE $ \lambda^* \leftarrow \lambda^{*\prime}, \quad  \bar{\varpi}_{\lambda^*} \leftarrow \bar{\varpi}_{\lambda^{*\prime}} $
		\STATE $ g \leftarrow  {\frac{\partial \bar{\varpi}_{k^*}}{\partial  \lambda}(\lambda^{*\prime})}^\top $
		\STATE $ \alpha \leftarrow 1.2\alpha $  \COMMENT{increase the stepsize}
		\ELSE 
		\STATE $ \alpha \leftarrow 0.5\alpha $  \COMMENT{decrease the stepsize}
		\ENDIF
		\UNTIL{ $|\lambda^{*\prime} - \lambda^*| < \epsilon$ }
		\IF{ $ \lambda^* > \max(\mathbb{I}_{k*})$  }
		\STATE $ k^* \leftarrow k^* + 1 $
		\STATE GOTO Line 1
		\ELSIF{  $ \lambda^* < \min(\mathbb{I}_{k*})$  }  
		\STATE $ k^* \leftarrow k^* - 1 $
		\STATE GOTO Line 1
		\ENDIF
	\end{algorithmic}
\end{algorithm}

Algorithm \ref{Algorithm:grad_descent} is ran at each time step in
order to calculate the coordinate transformation. In Algorithm
\ref{Algorithm:grad_descent}, lines 1 to 11 are the familiar monotonic
gradient descent algorithm with stepsize adaptation
\cite{marc2012lecture}. Lines 12 to 18 determine which spline we are
on. Depending on the path following application, if there is a known
neighbourhood of the tangential velocity, the stepsize $\alpha$ in
Algorithm \ref{Algorithm:grad_descent} can be tuned in order to
minimize the number of steps taken in the gradient descent and thus
the computational load.

\section{Supporting proofs and Derivations}
\begin{proof}[Proof of Lemma~\ref{lemma:linear_independent_diff}
  \label{appendix:proof_lemma_LIDIFF}]  Note that \eqref{eqn:lambda_star} is solved for by Algorithm \ref{Algorithm:grad_descent}, so if the output $y$ is equidistant to multiple points on the path, then a $\lambda^*$ is always well defined by the local search. 
  
   Next we write the
  differentials for each function.
	\begin{align*}
		\frac{\partial \eta_1}{\partial x} &= \begin{bmatrix}
					\frac{d s}{d \lambda}  \frac{d \varpi}{d y} J(x_c) & ,& 0_{1 \times N}
									 			\end{bmatrix}	\\
						  		 			&= \begin{bmatrix} 
					\vect{e}_1(\lambda^*)^\top J(x_c) & ,& 0_{1 \times N}
									 			\end{bmatrix}	\\
		\frac{\partial \eta_2}{\partial x} &= \begin{bmatrix}
					* & , & \vect{e}_1(\lambda^*)^\top J(x_c)
									 			\end{bmatrix}	\\
		\frac{\partial \xi_1^{j-1}}{\partial x} &= \left[
			\left( \left( h(x)-\sigma(\lambda^*) \right)^\top \left( \frac{1}{\norm{\sigma'(\lambda^*)}}\vect{e}_j'(\lambda^*)\vect{e}_1(\lambda^*)^\top \right)
			\right.\right. \\ & \qquad \qquad + \left.\left.\left( \vect{e}_j(\lambda^*)^\top \right) \right) J(x_c) \quad,  \quad 0_{1 \times N} \right]
									 			\\
		\frac{\partial \xi_2^{j-1}}{\partial x} &= \left[
			* \quad , \quad \left( \left( h(x)-\sigma(\lambda^*) \right)^\top \left( \frac{1}{\norm{\sigma'(\lambda^*)}}\vect{e}_j'(\lambda^*)\vect{e}_1(\lambda^*)^\top \right)\right.\right.  \\ & \qquad \qquad  + \left.\left.\left( \vect{e}_j(\lambda^*)^\top \right) \right) J(x_c) \right]				 										 	 			
	\end{align*}
	where $*$ denotes a vector of dimension $1 \times N$. Looking at the transversal terms, the product $\frac{1}{\norm{\sigma'(\lambda^*)}}\left( h(x)-\sigma(\lambda^*) \right)^\top \left( \vect{e}_j'(\lambda^*) \right)$ produces a scalar, call it $\alpha$. Thus the terms become  
	\begin{align*}
		\frac{\partial \xi_1^{j-1}}{\partial x} &= \begin{bmatrix}
			\left( \alpha\vect{e}_1(\lambda^*)^\top 
			+  \vect{e}_j(\lambda^*)^\top \right) J(x_c) & ,& 0_{1 \times N}
									 			\end{bmatrix}	\\
		\frac{\partial \xi_2^{j-1}}{\partial x} &= \begin{bmatrix}
			*& , & \left( \alpha\vect{e}_1(\lambda^*)^\top 
			+ \vect{e}_j(\lambda^*)^\top \right) J(x_c) 
									 			\end{bmatrix}	
	\end{align*}
	Since the domain $U$ does not include singularity points of $J(x_c)$ by the class of systems, the differentials will be linearly independent if the vectors $\vect{e}_i, i=1,...,p$ are orthogonal. These vectors are orthogonal because they are the FS orthonormal basis vectors.
\end{proof}

\begin{proof}[Derivation of \eqref{eq:eta2}]
\label{appendix:derivation_of_eta2}
By the fundamental theorem of calculus $\frac{\mathrm{d} s_{k^*}}{\mathrm{d} \lambda^*} = \norm{\sigma_{k^*}'(\lambda^*)}$. Geometric arguments provide that $\frac{\mathrm{d} \varpi_{k^*}}{\mathrm{d} y} = K_{k^*}(y)\sigma_{k^*}'(\lambda^*)^\top$ where $K_{k^*} : \Real^p \rightarrow \Real $ is a smooth scalar function \cite{consolini2010path}. To solve for $K_{k^*}(y)$, the following identity may be used:
\begin{equation*}
	\varpi_{k^*}(y) = \varpi_{k^*}(\sigma_{k^*}(\varpi_{k^*}(y))).
\end{equation*}
Differentiating both sides:
\begin{align*}
	\begin{split}
		& \frac{\mathrm{d} \varpi_{k^*}}{\mathrm{d} y} = \frac{\mathrm{d} \varpi_{k^*}}{\mathrm{d} y}\left.\sigma'(\lambda)\right|_{\lambda=\lambda^*}\frac{\mathrm{d} \varpi_{k^*}}{\mathrm{d} y} \\
		\Rightarrow & \frac{\mathrm{d} \varpi_{k^*}}{\mathrm{d} y}\left.\sigma_{k^*}'(\lambda)\right|_{\lambda=\lambda^*} = 1 \\
		\Rightarrow & \frac{\mathrm{d} \varpi_{k^*}}{\mathrm{d} y}\left.\sigma_{k^*}'(\lambda)\right|_{\lambda=\lambda^*} = \left(K_{k^*}(y)\sigma_{k^*}'(\lambda^*)^\top\right)\sigma_{k^*}'(\lambda^*) = 1 \\
	\Rightarrow & K_{k^*}(y) = \frac{1}{\norm{\sigma_{k^*}'(\lambda^*)}^2}.
	\end{split}
\end{align*}
Thus $\eta_2$ can be compactly written as \eqref{eq:eta2}.
\end{proof}

\begin{proof}[Derivation of \eqref{eq:transversal_vel}]
\label{appendix:derivation_of_xi_vel}
First note the identity $\frac{ \mathrm{d}}{ \mathrm{d} t} \vect{e}_j(\lambda^*)=\vect{e}_j'(\lambda^*)\frac{ \mathrm{d} \varpi}{ \mathrm{d} y}\dot{y}$, and $\frac{ \mathrm{d} \varpi}{ \mathrm{d} y} = \frac{\sigma'(\lambda^*)^\top}{\norm{\sigma_{k^*}'(\lambda^*)}^2}$ from above. Then, 
\begin{multline*}
	\xi_2^{j-1} \coloneqq \dot{\xi}_1^{j-1} =  \left( \vect{e}_j'(\lambda^*)
      	                              \left. \frac{ \mathrm{d} \varpi}{ \mathrm{d} y} \right|_{y=h(x)}
									  J(x_c)x_v\right)^\top (h(x) - \sigma(\lambda^*))
	\\+ \vect{e}_j(\lambda^*)^\top \left( J(x_c)x_v - \sigma'(\lambda^*) 
									  \left. \frac{ \mathrm{d} \varpi}{ \mathrm{d} y} \right|_{y=h(x)}
									  J(x_c)x_v \right)
	\\=\left( \vect{e}_j'(\lambda^*) \frac{\sigma'(\lambda^*)^\top}{\norm{\sigma'(\lambda^*)}^2}
									  J(x_c)x_v\right)^\top (h(x) - \sigma(\lambda^*))
	\\+ \vect{e}_j(\lambda^*)^\top \left( J(x_c)x_v - 
									  \vect{e}_1(\lambda^*)\vect{e}_1(\lambda^*)^\top
									  J(x_c)x_v \right)
\end{multline*}
and using the simplification that $\left\langle \vect{e}_j, \vect{e}_1 \right\rangle = 0$ for $j>1$ and $\eta_2(x) = \left\langle \vect{e}_1, J(x_c)x_v \right\rangle$ yields \eqref{eq:transversal_vel}.
\end{proof}

\begin{proof}[Generalized FS equations \cite{akivis2011elie}]
\label{appendix:generalized_FS_equations}
	\begin{multline} \label{eq:generalized_FS_equations} \tag{A.2}
		\begin{bmatrix}
		\vect{e}_1'(\lambda) \\
		\vect{e}_2'(\lambda) \\
		\vdots \\
		\vect{e}_{p-1}'(\lambda) \\
		\vect{e}_p'(\lambda)
		\end{bmatrix} = \norm{\sigma'(\lambda)} \\
		\begin{bmatrix}
		0 & \mathcal{X}_1(\lambda) & \dots & 0 & 0 \\
		-\mathcal{X}_1(\lambda) & 0 & \dots & 0 & 0 \\
		\vdots & \vdots & \ddots & \vdots & \vdots \\
		0 & 0 & \dots & 0 & \mathcal{X}_{p-1}(\lambda) \\
		0 & 0 & \dots & -\mathcal{X}_{p-1}(\lambda) & 0 
		\end{bmatrix}\begin{bmatrix}
		\vect{e}_1(\lambda) \\
		\vect{e}_2(\lambda) \\
		\vdots \\
		\vect{e}_{p-1}(\lambda) \\
		\vect{e}_p(\lambda)
		\end{bmatrix}
    \end{multline}
    where  $\mathcal{X}_i(\lambda) = \frac{\left\langle
                  \vect{e}_{i}'(\lambda), \vect{e}_{i+1}(\lambda)
                \right\rangle} {\norm{\sigma'(\lambda)}}$, $i \in \{1,...,p-1\}$.
\end{proof}

\endgroup
\end{appendices}

\bibliographystyle{IEEEtran}
\bibliography{IEEEabrv,mybib}

\end{document}